\documentclass[]{article}
\usepackage{proceed2e}

\usepackage[T1]{fontenc}
\usepackage[utf8x]{inputenc}
\usepackage[english]{babel}

\usepackage[numbers]{natbib}
\usepackage{times}

\usepackage{amsmath}
\usepackage{amssymb}
\usepackage{amsthm}
\usepackage{algorithm}
\usepackage{algorithmic}
\usepackage{float}
\newfloat{Subroutine}{thp}{lop}
\usepackage{eqparbox} 
\usepackage{placeins}

\usepackage{setspace}
\usepackage{xspace}

\newcommand{\todod}[1]{}
\newcommand{\todom}[1]{}
\newcommand{\todomi}[1]{}
\newcommand{\todoa}[1]{}
\newcommand{\todoaout}[1]{}

\usepackage[colorlinks,
           linkcolor=red,
           citecolor=blue,
           urlcolor=magenta,
           linktocpage,
           plainpages=false]{hyperref}

\usepackage{xcolor} 
\usepackage{wrapfig}
\usepackage{array}
\usepackage{enumitem}
\usepackage{mdframed}
\usepackage{bm}
\usepackage{amsfonts} 

\def\:#1{\protect \ifmmode {\mathbf{#1}} \else {\textbf{#1}} \fi}

\newcommand{\indfunc}{\mathbb{I}}
\newcommand{\condbar}{\;\middle|\;}
\renewcommand{\epsilon}{\varepsilon}
\renewcommand{\Re}{\mathbb{R}}
\newcommand{\X}{\mathcal{X}}
\newcommand{\bphi}{\boldsymbol{\phi}}
\newcommand{\nystrom}{Nystr\"{o}m\xspace}
\newcommand{\inkoracle}{\texttt{\textcolor[rgb]{0.5,0.2,0}{INK-Oracle}}\xspace}
\newcommand{\inkestimate}{\texttt{\textcolor[rgb]{0.5,0.2,0}{INK-Estimate}}\xspace}
\newcommand{\batchexact}{\texttt{Batch-Exact}\xspace}
\newcommand{\directsample}{\texttt{Direct-sample}\xspace}

\newcommand{\shrinkexpandop}{\texttt{\textcolor[rgb]{0.5,0.2,0}{Shrink-Expand}}\xspace}
\newcommand{\shrinkop}{\texttt{\textcolor[rgb]{0.5,0.2,0}{Shrink}}\xspace}
\newcommand{\expandop}{\texttt{\textcolor[rgb]{0.5,0.2,0}{Expand}}\xspace}

\newcommand{\wt}[1]{\widetilde{#1}}
\newcommand{\wh}[1]{\widehat{#1}}

\newcommand{\wb}[1]{\overline{#1}}
\newcommand{\transp}{\mathsf{T}}
\DeclareMathOperator*{\argmin}{arg\,min}
\DeclareMathOperator*{\argmax}{arg\,max}
\DeclareMathOperator*{\Tr}{Tr}

\DeclareMathOperator*{\Diag}{Diag}

\DeclareMathOperator*{\Rank}{Rank}

\newcommand{\norm}[2]{\left\Vert #1 \right\Vert_{#2}}
\newcommand{\normsmall}[1]{\Vert #1 \Vert}

\newcommand{\probability}{\mathbb{P}}

\DeclareMathOperator*{\expectedvalue}{\mathbb{E}}

\newcommand{\Real}{\mathbb{R}}

\newcommand{\kerfunc}{\mathcal{K}}
\newcommand{\kermatrix}{{\:K}}
\newcommand{\featkermatrix}{{\:\Phi}}

\newcommand{\deff}[1]{d_{\text{eff}}(#1)}
\newcommand{\adeff}[1]{\wt{d}_{\text{eff}}(#1)}
\newcommand{\atau}{\wt{\tau}}
\newcommand{\akermatrix}{\:{\wt{K}}}
\newcommand{\bkermatrix}{\:{\wb{K}}}

\newcommand{\selmatrix}{{\:S}}

\newcommand{\dataset}{\mathcal{D}}

\newcommand{\bigotime}{\mathcal{O}}

\newcommand{\R}{\mathcal{R}}

\newtheorem{theorem}{Theorem}
\newtheorem{corollary}{Corollary}
\newtheorem{definition}{Definition}
\newtheorem{lemma}{Lemma}
\newtheorem{proposition}{Proposition}

\title{Analysis of Nystr\"om Method with Sequential Ridge Leverage Score Sampling}

\author{ {\bf Daniele Calandriello} \\
SequeL team \\
INRIA Lille - Nord Europe \\
\And
{\bf Alessandro Lazaric}  \\
SequeL team \\
INRIA Lille - Nord Europe \\
\And
{\bf Michal Valko}   \\
SequeL team \\
INRIA Lille - Nord Europe \\
}

\begin{document}

\maketitle

\begin{abstract}
    Large-scale kernel ridge regression (KRR) is limited by the need to store
    a large kernel matrix~$\kermatrix_t$. To avoid storing the entire matrix $\kermatrix_t$, Nystr\"{o}m methods subsample a
    subset of columns of the kernel matrix, and efficiently find an
    approximate KRR solution on the reconstructed $\akermatrix_t$. 
    The chosen subsampling  distribution in turn affects the 
    statistical and computational tradeoffs. For KRR problems, \citep{rudi2015less,alaoui2014fast}
    show that a sampling distribution proportional to the \emph{ridge leverage scores} (RLSs) 
    provides strong reconstruction guarantees for $\akermatrix_t$. While exact RLSs are as difficult to
    compute as a KRR solution, we may be able to approximate them well enough.
    In this paper, we study KRR problems in a sequential setting and introduce the \inkestimate algorithm,
    that \textit{incrementally} computes the RLSs estimates. \inkestimate maintains a
    small \textit{sketch} of $\kermatrix_t$, that at each
    step is used to compute an intermediate estimate of the RLSs. 
    First, our sketch update does not
    require access to previously seen columns, and therefore a \textit{single pass} over
    the kernel matrix is sufficient. Second, the
    algorithm requires a fixed, small space budget to run dependent only on the
    \textit{effective dimension} of the kernel matrix. Finally, our sketch provides
    strong approximation guarantees on the distance
    $\normsmall{\kermatrix_t - \akermatrix_t}_2$, and on the statistical risk
    of the approximate KRR solution at \emph{any time}, because all our guarantees hold at any intermediate step.

\end{abstract}

\section{INTRODUCTION}\label{sec:intro}

Kernel ridge regression~\citep{scholkopf2001learning,shawe2004kernel} (KRR)
is a common nonparametric regression method with well studied theoretical
advantages. Its main drawback is that, for $n$ samples, storing and manipulating the
kernel regression matrix $\kermatrix_n$ requires $\bigotime(n^2)$ space,
and can become quickly intractable when $n$ grows.
This includes batch large scale KRR, and online KRR, where the size of the
dataset $t$ grows over time as new samples are added to the problem.
For this purpose, many different methods~\citep{williams2001using,engel2002sparse,kivinen2004online,rahimi2007random,le2013fastfood,zhang2015divide} attempt to reduce the memory required to store the kernel matrix, while still
producing an accurate solution.

For the batch case, the Nystr\"{o}m family of algorithms randomly selects a
subset of $m$ columns from the kernel matrix~$\kermatrix_n$ that are used to construct a
low rank approximation~$\akermatrix_{t}$ that requires only $\bigotime(nm)$ space to store.
The low-rank matrix is then used to find an approximate solution to the KRR
problem.
The quality of the approximate solution is strongly affected by the sampling
distribution and the number of columns selected \citep{rudi2015less}.
For example, uniform sampling is an approach with little computational overhead,
but does not work well for datasets with high coherence
\cite{gittens_revisiting_2013}, where the columns are weakly correlated.
In particular, \citet{bach_sharp_2012} shows that the number of columns
$m$ necessary for a good approximation when sampling uniformly scales
linearly with the maximum degree of freedoms of the kernel matrix.
In linear regression, the notion of coherence is strongly related to the definition
of leverage points or \emph{leverage scores }of the dataset \cite{everitt2002cambridge},
where points with high (statistical) leverage score are more influential in the regression
problem. For KRR, \citet{alaoui2014fast}
introduce a similar concept of \emph{ridge} leverage scores (RLSs) of a square
matrix, and shows that \nystrom approximations sampled according to RLS
have strong reconstruction guarantees of the form $\normsmall{\kermatrix_n - \akermatrix_n}_{2}$,
that translate into good guarantees for the approximate KRR solution
\cite{alaoui2014fast,rudi2015less}.
Compared to the uniform distribution, a distribution based on RLSs
better captures non-uniformities in the data,
and can achieve good approximations using only a number of columns $m$,
proportional to the average degrees of freedom of the matrix, called the
\textit{effective dimension} of the problem.
The disadvantage of RLSs compared to uniform sampling is the
high computational cost of exact RLSs, which is comparable to solving
KRR itself. \citet{alaoui2014fast} reduces this problem by showing that a distribution based on approximate RLSs
can also provide the same strong guarantees, if the RLSs are approximated
up to a constant error factor. They provide a fast method to compute these
RLSs, but,  unlike our approach, requires multiple passes over data.
Another disadvantage of their approach, that we address, is the \emph{inverse dependence on the minimal eigenvalue} 
of the kernel matrix in the error bound of \citet{alaoui2014fast}, which can be significant.

While Nystr\"{o}m methods are a typical choice in a batch setting, \emph{online kernel sparsification}
(OKS) \citep{engel2002sparse,engel2004kernel} examines each sample in the dataset
sequentially.  OKS maintains a small \textit{dictionary} of
relevant samples. Whenever a new sample
arrive, if the dictionary is not able to accurately represent the new sample
as a combination of the samples already stored, the dictionary is updated. This dictionary can be used to approximate KRR incrementally.
OKS decides whether to include a sample using the correlation between samples in the dictionary and the new sample.
This can measured using approximate linear dependency (ALD) \cite{engel2004kernel},
coherence \cite{richard_online_2009}, or the
surprise criterion~\cite{liu_information_2009}.

Generalization properties of online kernel sparsification were studied
by \citet{engel2004kernel}, but depend on the empirical error and are not
compared with an exact KRR solution on the whole dataset.
Online kernel regression with the ALD rule was analyzed by~\citet{sun2012size},
under the assumption that, asymptotically in $n$,
the eigenvalues of the kernel matrix decay exponentially fast. \citet{sun2012size} show that
in this case the size of the dictionary grows sublinearly in $t$, or in other
words that, asymptotically in $n$, the dictionary size converge to a fraction
of $n$ that will be small whenever the eigenvalues decay fast enough.
This space guarantee is weaker than the fixed space requirements of \nystrom
methods, one of the reasons is that these methods (unlike ours)
cannot remove a sample from the dictionary after inclusion.
Furthermore,  \citet{van2012kernel} studies variants of online kernel regression with a forgetting
factor for time-varying series, but these methods are not well studied in the normal KRR setting.
Unlike in the batch setting, in the sequential setting we often require the 
guarantees not only at the end but also \emph{in the intermediate steps} and this is our objective.
Inspired by the advances in the analyses of the Nystr\"om methods, in this paper, we focus on finding a
space efficient algorithm capable of solving KRR problems in the sequential setting
but that would be also equipped with generalization guarantees.
\paragraph{Main contributions}
We propose the \inkestimate algorithm that processes a dataset $\dataset$ of size $n$ in \emph{a single pass}. 
It requires only a small, fixed space budget, $\wb{q}$
proportional to the effective dimension of the problem and on the accuracy required.
The algorithm maintains a \nystrom approximation $\akermatrix_t$,
of the kernel matrix at time $t$, $\kermatrix_t$, based on RLSs estimates.
At each step, it uses only the approximation and the newly received sample to 
incrementally update the RLSs estimate, and to compute $\akermatrix_{t+1}$.
Unlike in the batch \nystrom setting, our challenge is to track RLSs and
an effective dimension that \emph{changes over time}. Sampling distributions
based on RLSs can become obsolete and biased, but we show 
how to update them over time \emph{without necessity of accessing previously seen
samples} outside of the ones contained in $\akermatrix_t$.
Our space budget~$\wb{q}$ scales with the
average degree of freedom of the matrix, 
and not the larger maximum degree of freedom (as by \citet{bach_sharp_2012}),
and does not imposes assumptions on the ridge regularization parameter, or
on the smallest eigenvalue of the problem as the result of \citet{alaoui2014fast}.
However, we provide the same strong guarantees as batch RLSs based \nystrom methods on
$\normsmall{\kermatrix_n - \akermatrix_n}_{2}$ and on the risk of the approximate
KRR solution.
In addition to batch \nystrom methods, all of these guarantees hold at any intermediate step $t$, and therefore
the algorithm can output \emph{accurate intermediate solutions}, or it can be interrupted
at \textit{any time} and return a solution with guarantees.
Finally, it operates in a \emph{sequential} setting, requiring only a single pass
over the data.

If we compare \inkestimate to other online kernel regression methods (such as OKS), our algorithm provides
generalization guarantees with respect to the exact KRR solution.
Furthermore, it provides a new criteria for inclusion of a sample in the dictionary,
in particular the ridge leverage scores. This criterion gives us a procedure
that not only randomly includes samples in the dictionary, but that also randomly discards them
to satisfy space constraints not only asymptotically, but at every step.

\section{BACKGROUND}\label{sec:setting}

In this section we introduce the notation used through the paper and we introduce the kernel ridge regression problem~\citep{scholkopf2001learning} and Nystr\"{o}m approximation of the kernel matrix with ridge leverage scores.

\textbf{Notation.}
We use curly capital letters $\mathcal{A}$ for collections. We use upper-case bold letters $\:A$ for matrices, lower-case bold letters $\:a$ for vectors, and lower-case letters $a$ for scalars. We denote by $[\:A]_{ij}$ and $[\:a]_i$ the $(i,j)$ element of a matrix and $i$th element of a vector respectively. We denote by $\:I_n\in\Re^{n\times n}$ the identity matrix of dimension $n$ and by $\Diag(\:a)\in\Re^{n\times n}$ the diagonal matrix with the vector $\:a\in\Re^n$ on the diagonal. We use $\:e_{i,n} \in \Re^{n}$ to denote the indicator vector for element $i$ of dimension $n$. When the dimensionality of $\:I$ and $\:e_{i}$ is clear from the context, we omit the $n$. We use $\:A \succeq \:B$ to indicate that $\:A-\:B$ is a PSD matrix. Finally, the set of integers between 1 and $n$ is denoted by $[n] := \{1,\ldots,n\}$.

\subsection{Exact Kernel Ridge Regression}\label{sec:extact.krr}

\textbf{Kernel regression.}
We consider a regression dataset $\dataset = \{(\:x_t,y_t)\}_{t=1}^n$, with input $\:x_t \in \X \subseteq \Re^d$ and output $y_t\in\Re$. We denote by $\kerfunc: \X \times \X \rightarrow \Re$ a positive definite kernel function and by $\varphi: \X \rightarrow \Re^D$ the corresponding feature map,%
\footnote{where $D$ can be very large or infinite (e.g. gaussian kernel)} 
so that the kernel is obtained as $\kerfunc(\:x, \:x') = \varphi(\:x)^\transp \varphi(\:x')$. Given the dataset $\dataset$, we define the kernel matrix $\kermatrix_t \in \Re^{t\times t}$ constructed on the first $t$ samples as the application of the kernel function on all pairs of input values, i.e., $[\kermatrix_t]_{ij} = \kerfunc(\:x_i, \:x_j)$ for any $i,j\in [t]$ and we denote by $\:y_t\in\Re^t$ the vector with components $y_i$, $i\in [t]$. We also define the feature vectors $\bphi_t = \varphi(\:x_t) \in \Re^D$ and after introducing the feature matrix
    \begin{align*}
    \featkermatrix_t =
    \left[
    \begin{array}{c|c|c|c}
        \bphi_1 & \bphi_2 & \dots & \bphi_t
    \end{array}
    \right] \in \Real^{D \times t},
    \end{align*}
we can rewrite the kernel matrix as $\kermatrix_t = \featkermatrix_t^\transp \featkermatrix_t$.
Whenever a new point $\:x_{t+1}$ arrives, the kernel matrix $\kermatrix_{t+1} \in \Real^{t+1 \times t+1}$ is obtained by bordering $\:K_t$ as
\begin{align}\label{eq:ker.bordering}
\kermatrix_{t+1} = \left[
    \begin{array}{c|c}
        \kermatrix_t & \wb{\:k}_{t+1} \\
    \hline
    \wb{\:k}_{t+1}^\transp & k_{t+1}
    \end{array}
    \right]
\end{align}
where $\wb{\:k}_{t+1} \in \Real^{t}$ is such that $[\wb{\:k}_{t+1}]_i = \kerfunc(\:x_{t+1}, \:x_i)$ for any $i\in [t]$ and $k_{t+1} = \kerfunc(\:x_{t+1}, \:x_{t+1})$. According to the definition of the feature matrix $\featkermatrix_t$, we also have $\:k_{t+1}~=~\featkermatrix_t^\transp\:\bphi_{t+1}$.

At any time $t$, the objective of \emph{sequential} kernel regression is to find the vector $\wh{\:w}_t \in \Real^t$ that minimizes the regularized quadratic loss
\begin{align}\label{eq:kern-regr-problem}
    \wh{\:w}_t = \argmin_{\:w}\normsmall{\:y_t - \:K_t \:w}^2 + \mu \normsmall{\:w}^2,
\end{align}
where $\mu\in\Re$ is a regularization parameter. 
This objective admits the closed form solution
\begin{align}\label{eq:kern-regr-exact-sol}
    \wh{\:w}_t = (\:K_t + \mu \:I)^{-1} \:y_t.
\end{align}
In the following, we use $\:K_t^\mu$ as a short-hand for $(\:K_t + \mu \:I)$.
In batch regression, $\wh{\:w}_n$ is computed only once when all the samples of
$\dataset$ are available, solving the linear system in
Eq.~\ref{eq:kern-regr-exact-sol} with $\:K_n$.
In the \emph{fixed-design} kernel regression, the accuracy
of resulting solution $\wh{\:w}_n$ is measured by the prediction error on the
input set from $\dataset$. More precisely, the prediction of the estimator
$\wh{\:w}_n$ in each point is obtained as $[\:K_n \wh{\:w}_n]_i$, while the
outputs $y_i$ in the dataset are assumed to be a noisy observation of an
unknown target function $f^*: \mathcal X \rightarrow \Re$, evaluated in $x_i$
i.e., for any $i\in[n]$,
\begin{align*}
y_i = f^*(x_i) + \eta_i,
\end{align*}
where $\eta_i$ is a zero-mean i.i.d.\ noise with bounded variance~$\sigma^2$. Let $\:f^*\in\Re^n$ be the vector with components $f^*(x_i)$, then the risk of $\wh{\:w}_n$ is measured as 
\begin{align}\label{eq:risk}
\R(\wh{\:w}_n) = \mathbb{E}_\eta\big[ || \:f^* - \:K_n \wh{\:w}_n ||_2^2\big].
\end{align}
If the regularization parameter $\mu$ is properly tuned, it is possible to show that $\wh{\:w}_n$ has near-optimal risk guarantees (in a minmax sense). 
Nonetheless, the computation of~$\wh{\:w}_n$ requires $\bigotime(n^3)$ time and $\bigotime(n^2)$ space, which becomes rapidly unfeasible for large datasets.

\subsection{Nystr\"{o}m Approximation with Ridge Leverage Scores}\label{sec:setting-bis}

A common approach to reduce the complexity of kernel regression is to (randomly) select a subset of $m$ samples out of $\dataset$, and compute the kernel between two points only when one of them is in the selected subset. This is equivalent to selecting a subset of columns of the $\kermatrix_n$ matrix. More formally, given the $n$ samples in $\dataset$, a probability distribution $\:p_n = [p_{1,n}, \ldots, p_{n,n}]$ is defined over all columns of $\kermatrix_n$ and $m\leq n$ columns are randomly sampled with replacement according to $\:p_n$. We define by $\mathcal{I}_{n}$ the sequence of $m$ indices $i \in [n]$ selected by the sampling procedure.
From $\mathcal{I}_{n}$, we construct the corresponding selection matrix $\selmatrix_n \in \Real^{n \times m}$, where each column $[\selmatrix_n]_{:,t}\in\Re^n$ is all-zero except from the entry corresponding to the $t$-th element in $\mathcal{I}_{n}$ (i.e., $[\selmatrix]_{ij}$ is non-zero if at trial $j$ the element~$i$ is selected).
Whenever the non-zero entries of $\selmatrix_n$ are set to $1$, sampling $m$
columns from matrix $\kermatrix_n$ is equivalent to computing $\kermatrix_n
\selmatrix_n \in \Real^{n \times m}$. More generally,
the non-zero entries of $\selmatrix_n$ could be set to some arbitrary
weight $[\selmatrix]_{ij} = b_{ij}$. The resulting regularized Nystr\"{o}m
approximation of the original kernel $\kermatrix_n$ is defined as
\begin{align}\label{eq:nystrom}
    \wt{\kermatrix}_n = \kermatrix_n \selmatrix_n(\selmatrix_n^\transp \kermatrix_n \selmatrix_n + \gamma\:I_m)^{-1}\selmatrix_n^\transp \kermatrix_n,
\end{align}
where $\gamma$ is a regularization term (possibly different from~$\mu$). At this point, $\wt{\kermatrix}_n$ can be used to solve Eq.~\ref{eq:kern-regr-exact-sol}. Let $\:W = (\selmatrix_n^\transp \kermatrix_n \selmatrix_n + \gamma\:I_m)^{-1} \in \Real^{m \times m}$ and $\:C = \kermatrix_n \selmatrix_n \:W^{1/2} \in \Real^{n \times m}$, applying the Woodbury inversion formula \cite{higham2002accuracy} we have
\begin{align}
    \wt{\:w}_n =& (\akermatrix_n + \mu \:I_{n})^{-1} \:y_n = (\:C\:I_{m}\:C^\transp + \mu \:I_{n})^{-1} \:y_n\nonumber\\
    =&\left(\frac{1}{\mu} \:I_n - \frac{1}{\mu^2}\:I_n\:C\left(\:I_m + \frac{1}{\mu}\:C^\transp \:C\right)^{-1}\:C^\transp\:I_{n}\right)\:y_n\nonumber\\
    =&\frac{1}{\mu} \left(\:y_n - \:C\left( \:C^\transp \:C + \mu\:I_m\right)^{-1}\:C^\transp\:y_n\right).\label{eq:linear-system-transformed}
\end{align}
Computing $\:W^{1/2}$ and $\:C$ takes $\bigotime(m^3)$ and $\bigotime(nm^2)$ time
using a singular value decomposition,
and so does solving the linear system. All the operations
require to store at most an $n \times m$ matrix. Therefore the final
complexity is reduced from $\bigotime(n^3)$ to $\bigotime(nm^2 + m^3)$ time,
and from $\bigotime(n^2)$ to $\bigotime(nm)$ space.
\citet{rudi2015less} recently showed that in random design, the risk of the
resulting solution $\wt{\:w}_n$ strongly depends on the choice of
$m$ and the column sampling distribution $\:p_n$.
Early methods sampled columns uniformly,
and \citet{bach_sharp_2012} shows that the using this distribution
can provide a good approximation when the maximum diagonal entry of
$\kermatrix_n(\kermatrix_n + \mu\:I)^{-1}$ is small.
Following on this approach, \citet{alaoui2014fast} propose a distribution proportional
to these diagonal entries and calls them $\gamma$-Ridge Leverage Scores.
We now restate their definition of RLS, corresponding sampling distribution, and the effective
dimension.
\begin{definition}\label{def:exact-lev-scores}
Given a kernel matrix $\:K_n\in\Re^{n\times n}$, the $\gamma$-ridge leverage score (RLS) of column $i\in [n]$ is
\begin{align}\label{eq:exact-rls}
    \tau_{i,n}(\gamma) = \:k_{i,n}^\transp (\kermatrix_n + \gamma\:I_m)^{-1} \:e_{i,n},
\end{align}
where $\:k_{i,n} = \kermatrix_n \:e_{i,n}$. Furthermore, the effective dimension $\deff{\gamma}_n$ of the kernel is defined as
\begin{align}\label{eq:exact-deff}
\deff{\gamma}_n = \sum_{i=1}^n \tau_{i,n}(\gamma) = \Tr\left(\kermatrix_n(\kermatrix_n + \gamma \:I_n)^{-1}\right).
\end{align}
The corresponding sampling distribution $\:p_n$ is defined as
\begin{align}\label{eq:exact-prob}
        [\:p_n]_{i} = p_{i,n} = \frac{\tau_{i,n}(\gamma)}{\sum_{j=1}^{n}\tau_{i,n}(\gamma)} = \frac{\tau_{i,n}}{\deff{\gamma}_n}.
\end{align}
\end{definition}

The RLSs are directly related to the structure of the kernel matrix and the
regularized regression. If we perform an eigendecomposition of the
kernel matrix as $\kermatrix_n= \:U_n \:\Lambda_n \:U_n^\transp$, then the RLS
of a column $i\in[n]$ is
\begin{align}\label{eq:exact-rls-2}
 \tau_{i,n}(\gamma) = \sum_{j=1}^n \frac{\lambda_j}{\lambda_j + \gamma}[\:U]_{i,j}^2,
\end{align}
which shows how the RLS is a weighted version of the standard leverage scores
(i.e., $\sum_j\:[U]_{i,j}^2$), where the weights depend on both the spectrum of
$\kermatrix_n$ and the regularization $\gamma$, which plays the role of a soft
threshold on the rank of $\kermatrix_n$. Similar to the standard leverage
scores \cite{drineas2011fast}, the RLSs measure the relevance of each point $\:x_i$ for the overall
kernel regression problem. Another interesting property of the RLSs is that
their sum is the effective dimension $\deff{\gamma}_n$, which measures
the intrinsic capacity of the kernel $\kermatrix_n$ when its spectrum is
soft-thresholded by a regularization $\gamma$.\footnote{Notice that indeed we
    have $\deff{\gamma}_n \leq \Rank(\kermatrix_n)$.} 
We refer to the overall Nystr\"{o}m method using RLS and sampling according to to $\:p_n$ in Eq.\,\ref{eq:exact-prob} as \texttt{Batch-Exact}, which is illustrated in
Alg.~\ref{alg:oracle-batch-direct}. We single out the $\directsample$ subroutine (which simply draws~$m$ independent samples from the multinomial distribution~$\:p_n$) to ease the introduction of our incremental algorithm in the next section.

With the following claim, \citet{alaoui2014fast} prove that the regularized \nystrom approximation $\akermatrix_n$ obtained from
Eq.\,\ref{eq:nystrom} guarantees an accurate reconstruction of the original kernel matrix
$\kermatrix_n$, and the risk 
of the associated solution
$\wt{\:w}_n$ is close to the risk of the exact solution $\wh{\:w}_n$.
\begin{algorithm}[t!]
    \begin{algorithmic}[1]
        \renewcommand{\algorithmicrequire}{\textbf{Input:}}
        \renewcommand{\algorithmicensure}{\textbf{Output:}}
        \REQUIRE {$\dataset$}, regularization parameter $\gamma$, sampling budget $m$ and probabilities $\:p_n$ (Eq.~\ref{eq:exact-prob})
        \ENSURE Nystr\"{o}m approximation $\akermatrix_n$, matrix $\selmatrix_n$
        \STATE Compute $\mathcal{I}_{n}$ using $\directsample(\:p_{n}, m)$ 
        \STATE Compute $\selmatrix_n$ using $\mathcal{I}_{n}$ and weights $1/{\sqrt{m p_{i,n}}}$
        \STATE Compute $\akermatrix_n$ using $\selmatrix_n$ and Eq.\,\ref{eq:nystrom}
    \end{algorithmic}
    \caption{\texttt{Batch-Exact} algorithm}
    \label{alg:oracle-batch-direct}
\end{algorithm}
\begin{Subroutine}[t!]
    \begin{algorithmic}[1]
        \renewcommand{\algorithmicrequire}{\textbf{Input:}}
        \renewcommand{\algorithmicensure}{\textbf{Output:}}
        \REQUIRE probabilities $\:p_n$, sampling budget $m$
        \ENSURE subsampled column indices $\mathcal{I}_{n}$
        \FOR{$j = \{1, \dots, m\}$} 
            \STATE Sample $i \sim \mathcal{M}(p_{1,n}, \dots, p_{n,n})$
            \STATE Add $i$ to $\mathcal{I}_{n}$
        \ENDFOR
    \end{algorithmic}
    \caption{$\directsample(\mathbf{p}_{n}, m) \rightarrow \mathcal{I}_{n}$}
    \label{subr:direct-sample}
\end{Subroutine}
\begin{proposition}[\citet{alaoui2014fast}, App. A, Lem. 1]\label{prop:alaoui-nyst-guar}
    Let $\gamma \geq 1$, let $\kermatrix_n$ be the full kernel matrix ($t=n$), and let
$\tau_{i,n}$, $\deff{\gamma}_{n}$, $p_{i,n}$ be defined according to
Definition \ref{def:exact-lev-scores}.
For any $0\leq \varepsilon \leq 1$,
and $0 \leq \delta \leq 1$, if we 
run Alg.\,\ref{alg:oracle-batch-direct} using $\directsample$
(Subroutine~\ref{subr:direct-sample}) with sampling budget $m$,
    \begin{align*}
        m \geq \left( \frac{2\deff{\gamma}}{\varepsilon^2}\right)\log\left(\frac{n}{\delta}\right),
    \end{align*}
to compute matrix $\selmatrix_n$,
then with probability $1-\delta$
the corresponding Nystr\"{o}m approximation $\akermatrix_n$ in Eq.~\ref{eq:nystrom} satisfies the condition
    \begin{align}\label{cond:alaoui-nyst-guar}
        \:0 \preceq \kermatrix_n - \wt{\kermatrix}_n \preceq  \frac{\gamma}{1-\varepsilon}\kermatrix_n(\kermatrix_n + \gamma\:I_n)^{-1} \preceq \frac{\gamma}{1-\varepsilon}\:I_n.
    \end{align}
%
Furthermore, replacing $\kermatrix_n$ by $\wt{\kermatrix}_n$ in Eq.~\ref{eq:kern-regr-exact-sol} gives an approximation solution $\wt{\:w}_n$ such that
    \begin{align*}
        \mathcal{R}(\wt{\:w}_n) \leq& \left(1 + \frac{\gamma}{\mu}\frac{1}{1-\varepsilon}\right)^{2}\mathcal{R}(\wh{\:w}_n).
    \end{align*}

\end{proposition}
\paragraph{Discussion}
This result directly relates the number of columns selected $m$ with the
accuracy of the approximation of the kernel matrix. In particular, the inequalities in Eq.\,\ref{cond:alaoui-nyst-guar} show that
the distance $\normsmall{\kermatrix_n - \akermatrix_n}_{2}$
is smaller than $\gamma/(1-\varepsilon)$. This level of accuracy is then sufficient to guarantee that,
when $\gamma$ is properly tuned,
the prediction error of $\wt{\:w}_n$ is only a factor $(1+2\epsilon)^2$
away from the error of the exact solution $\wh{\:w}$.
As it was shown in \cite{alaoui2014fast},
using $\akermatrix_n$ in place of $\kermatrix_n$
introduces a bias in the solution $\wt{\:w}_n$
of order $\gamma$. For appropriate choices of $\gamma$ this bias is dominated
by the ridge regularization bias controlled by $\mu$.
As a result,~$\wt{\:w}_n$ can indeed achieve almost the same risk as~
$\wh{\:w}_n$ and, at the same time, ignore all directions that are whitened by
the regularization and only approximate those that are more relevant for ridge
regression, thus reducing both time and space complexity.
The RLSs quantify how important each column is to approximate these relevant
directions but computing exact RLSs $\tau_{i,n}(\gamma)$ using Eq.\,\ref{eq:exact-rls}
is as hard as solving the regression problem itself.
Fortunately, in many cases it is computationally feasible to find an approximation of the
RLSs. \citet{alaoui2014fast} explore this possibility, showing that
the accuracy and space guarantees are robust to perturbations in the distribution
$\:p_n$, and provide a two-pass method to compute such approximations.
Unfortunately, the accuracy of their RLSs approximation is proportional to the
smallest eigenvalue $\lambda_{\min}(\kermatrix_n)$, which in some cases can be
very small.
In the rest of the paper, we propose an \emph{incremental} approach that requires
only a \emph{single pass} over the data and, at the same time,
does not depend on $\lambda_{\min}(\kermatrix_n)$ to be large as in \cite{alaoui2014fast},
or on $\max_{i} \tau_{i,n}$ to be small as in \cite{bach_sharp_2012}.

%
%


\begin{algorithm}[t!]
    \setstretch{1.2}
    \begin{algorithmic}[1]
        \renewcommand{\algorithmicrequire}{\textbf{Input:}}
        \renewcommand{\algorithmicensure}{\textbf{Output:}}
        \REQUIRE Dataset {$\dataset$}, regularization $\gamma$, sampling budget $\wb{q}$ and \texttt{$(\alpha,\beta)$-oracle}
        \ENSURE $\akermatrix_n$, $\selmatrix_n$
        \STATE Initialize $\mathcal{I}_0$ as empty, $\wt{p}_{1,0} = 1, b_{1,0} = 1$, budget $\wb{q}$
        \FOR{$t = 0,\dots,n-1$}
            \STATE Receive new column $\wb{\:k}_{t+1}$ and scalar $k_{t+1}$
            \STATE Receive $\alpha$-leverage scores $\atau_{i,t+1}$ for any $i \in \mathcal{I}_{t} \cup \{t+1\}$ from \texttt{$(\alpha,\beta)$-oracle}
            \STATE Receive $\beta$-approximate $\adeff{\gamma}_{t+1}$ from \texttt{$(\alpha,\beta)$-oracle}
            \STATE Set $\wt{p}_{i,t+1} = \min\{\atau_{i,t+1}/\adeff{\gamma}_{t+1}, \wt{p}_{i,t}\}$
            \STATE $\mathcal{I}_{{t+1}},\:b_{t+1}$ = $\shrinkexpandop(\mathcal{I}_{t}, \wt{\:p}_{t+1}, \:b_{t}, \wb{q})$
            \STATE Compute $\selmatrix_{t+1}$ using $\mathcal{I}_{t+1}$ and weights ${\sqrt{b_{i,t+1}}}$
            \STATE Compute $\akermatrix_{t+1}$ using $\selmatrix_{t+1}$ and Equation~\ref{eq:nystrom}
        \ENDFOR
        \STATE Return $\akermatrix_n$ and $\selmatrix_n$
    \end{algorithmic}
    \caption{The \inkoracle algorithm}
    \label{alg:ink-oracle-revamp}
\end{algorithm}

\section{INCREMENTAL ORACLE KERNEL APPROXIMATION WITH SEQUENTIAL SAMPLING}\label{sec:incremental-rejection}
Our main goal is to extend the known ridge leverage score sampling to the \emph{sequential setting}.
This comes with several challenges that needs to be addressed \emph{simultaneously:}
\begin{enumerate}
\item The RLSs change when a new sample arrives. We not only need to estimate them,
    but to \textit{update} this estimate over iterations.
\item The effective dimension $\adeff{\gamma}_{t}$, necessary to normalize the 
    leverage scores for the sampling distribution $\:p_n$,
    depends on the interactions of all columns, including the ones that we decided \emph{not} to keep.
\item Due to changes in RLSs, our  sampling distribution $\wt{\:p}_t$ \emph{changes over time}.
    We need to update to dictionary to reflect these changes, or it will quickly
    become \textit{biased}, but once we completely drop a column, we cannot sample it again.
\end{enumerate}

In this section, we address the third challenge of incremental updates of the columns with an algorithm for the approximation of the 
kernel matrix $\kermatrix_n$, assuming that the first and second issue 
are addressed by an \emph{oracle} giving \emph{both} good approximations of leverage scores and the effective dimension.

\begin{definition}\label{def:alpha-beta-approx}
At any step $t$, an \texttt{$(\alpha,\beta)$-oracle} returns an $\alpha$-approximate ridge leverage scores $\atau_{i,t}$ which satisfy
    \begin{align*}
        \frac{1}{\alpha} \tau_{i,t}(\gamma) \leq \atau_{i,t} \leq \tau_{i,t}(\gamma),
    \end{align*}
for any $i\in[t]$ and and a $\beta$-approximate effective dimension $\adeff{\gamma}_{t}$ which satisfy
    \begin{align*}
        \deff{\gamma}_{t} \leq \adeff{\gamma}_{t} \leq \beta\deff{\gamma}_{t}.
    \end{align*}
\end{definition}
We address the first and second challenge in Sect.\,\ref{sec:incremental-oracle}
with an efficient implementation and  \texttt{$(\alpha,\beta)$-oracle}.
In the following we give the \emph{incremental}  \inkoracle  algorithm
equipped with an \texttt{$(\alpha,\beta)$-oracle}
 that after $n$ steps it returns a kernel approximation with the same 
properties as if an \texttt{$(\alpha,\beta)$-oracle} was used directly at time $n$.

\subsection{The \inkoracle Algorithm}

\begin{Subroutine}[t!]
\setstretch{1.2}
    \begin{algorithmic}[1]
        \renewcommand\algorithmiccomment[1]{%
        \hfill \(\triangleright\)\eqparbox{COMMENT}{#1}}
        \renewcommand{\algorithmicrequire}{\textbf{Input:}}
        \renewcommand{\algorithmicensure}{\textbf{Output:}}
        \REQUIRE $\mathcal{I}_{t}$, app. pr. $\{(\wt{p}_{i,t+1}, b_{i,t}) : i \in \mathcal{I}_{t}\}$, $\wt{p}_{t+1, t+1}$, $\wb{q}$
        \ENSURE $\mathcal{I}_{{t+1}}$
        \FORALL[\shrinkop]{$i \in \{1,\dots,t\}$}
            \STATE $b_{i,t+1} = b_{i,t}$
            \WHILE{$b_{i,t+1}\wt{p}_{i,t+1} \leq 1/\wb{q}$ and $b_{i,t} \neq 0$}\label{alg:decision-rule-oracle}
                \STATE Sample a random Bernoulli $\mathcal{B}\left(\frac{b_{i,t+1}}{b_{i,t+1}+1}\right)$
                \STATE On success set $b_{i,t+1} = b_{i,t+1}+1$
                \STATE On failure set $b_{i,t+1} = 0$
            \ENDWHILE
        \ENDFOR
        \STATE $b_{t+1,t+1} = 1$\COMMENT{\expandop}
        \WHILE{$b_{t+1,t+1}\wt{p}_{t+1,t+1} \leq 1/\wb{q}$ and $b_{t+1,t+1} \neq 0$}
                \STATE Sample a random Bernoulli $\mathcal{B}\left(\frac{b_{t+1,t+1}}{b_{t+1,t+1}+1}\right)$
                \STATE On success set $b_{t+1,t+1} = b_{t+1,t+1}+1$
                \STATE On failure set $b_{t+1,t+1} = 0$
        \ENDWHILE
        \STATE Add to $\mathcal{I}_{t+1}$ all columns with $b_{i,t+1} \neq 0$
    \end{algorithmic}
    \caption{$\shrinkexpandop(\mathcal{I}_{t}, \wt{\mathbf{p}}_{t+1}, \mathbf{b}_{t}, \wb{q})$}
    \label{subr:shrinkop-revamp}
\end{Subroutine}

 Apart from an \texttt{$(\alpha,\beta)$-oracle} and the 
dataset $\dataset$, \inkoracle (Alg.\,\ref{alg:ink-oracle-revamp}) receives as input the regularization parameter 
$\gamma$ used in constructing the final \nystrom approximation and a 
sampling budget $\wb{q}$.
It initializes the index dictionary $\mathcal{I}_{0}$ of
stored columns as empty, and the estimated probabilities as $\wt{p}_{i,0} = 1$.
Finally it initializes a set of integer weights $b_{i,0} = 1$. These weights will represent
a discretized approximation of $1/\wt{p}_{i,t}$ (the inverse of the probabilities).
At each time step $t$, it receives a new column $\wb{\:k}_{t+1}$ and $k_{t+1}$.
This can be implemented either by having a separate algorithm, constructing each
column sequentially and stream it to \inkoracle, or by having \inkoracle
store just the samples (for an additional $\bigotime(td)$  space complexity)
and independently compute the column once.
The algorithm invokes the  \texttt{$(\alpha,\beta)$-oracle} to compute
approximate probabilities $\wt{p}_{i,t+1} = \atau_{i,t+1}/\adeff{\gamma}_{t+1}$,
and then takes the minimum $\min\{\wt{p}_{i,t+1}, \wt{p}_{i,t}\}$ for the sampling probability. As
our analysis will reveal, this step is necessary to ensure that the \shrinkexpandop operation
remains well defined, since the true probabilities $p_{i,t}$ decrease over time.
It is important to notice that differently from the batch sampling setting, the approximate probabilities do not necessarily sum to one, but it is guaranteed that $\sum_{i=1}^{t} \wt{p}_{i,t} \leq 1$.
The \shrinkexpandop procedure is composed of two steps.
In the \shrinkop step, we update the weights of the columns already in our dictionary.
To decide whether a weight should be increased or not,
the product of the weight at the preceding step $b_{i,t-1}$ and the new
estimate $\wt{p}_{i,t}$ is compared to a threshold. If the product is
above the threshold, it means the probability did not change much, and
no action is necessary. If the product falls below
the threshold, it means the decrease of $\wt{p}_{i,t}$ is significant,
and the old weight is not representative anymore and should be increased.
To increase the weight (e.g.~from $k$ to $k+1$), we draw a Bernoulli random
variable $\mathcal{B}(\frac{k}{k+1})$,
and if it succeeds we increase the weight to $k+1$, while if it fails
we set the weight to 0.
The more $\wt{p}_{i,t}$ decrease over time, the higher the chanches
that $b_{i,t+1}$ is set to zero, and the index $i$ (and the associated
column $\:k_{i,t+1}$) is completely dropped from the dictionary.
Therefore, the \shrinkop step randomly reduces the size of the dictionary
to reflect the evolution of the probabilities. Conversely, the \expandop
step introduces the new column in the
dictionary, and quickly updates its weight~$b_{t,t}$ to reflect~$\wt{p}_{t,t}$.
Depending on the relevance (encoded by the RLS) of the new column,
this means that it is possible that the new column is discarded at the same
iteration as it is introduced.
For a whole pass over the dataset, \inkoracle queries the oracle for each RLS at
least once, but it \emph{never} asks again for the RLS of a columns dropped from~$\mathcal{I}_{t}$.
As we will see in the next section, this greatly simplifies the construction of the oracle.
Finally, after updating the dictionary, we use the updated weights
$\sqrt{b_{i,t}}$ to update the approximation $\akermatrix_t$,
that can  be used at any time and not only in the end.
\subsection{Analysis of \inkoracle}
The main result of this section is the lower bound on the number 
of columns required to be kept in order to guarantee a $\gamma/(1-\varepsilon)$ approximation
of $\kermatrix_t$.
\begin{theorem}\label{thm:ink-oracle-main}
 Let $\gamma > 1$. Given access to an
\texttt{$(\alpha,\beta)$-oracle}, for $0\leq \varepsilon \leq 1$
and $0 \leq \delta \leq 1$, if we run Alg.\,\ref{alg:ink-oracle-revamp} with parameter $\wb{q}$
    \begin{align*}
        \wb{q} \geq\left( \frac{28 \alpha \beta\deff{\gamma}_t}{\varepsilon^2}\right)\log\left(\frac{4t}{\delta}\right),
    \end{align*}
to compute a sequence of random matrices $\selmatrix_t$ with a random number of columns $Q_t$,
then with probability $1-\delta$, for all $t$
the corresponding Nystr\"{o}m approximation $\akermatrix_t$ (Eq.~\ref{eq:nystrom}) satisfies condition in Eq.\,\ref{cond:alaoui-nyst-guar},
    \begin{align*}
        \:0 \preceq \kermatrix_t - \akermatrix_t \preceq  \frac{\gamma}{1-\varepsilon}\kermatrix_t(\kermatrix_t + \gamma\:I)^{-1} \preceq \frac{\gamma}{1-\varepsilon}\:I.
    \end{align*}
and the number of columns selected $Q_t$ is such that
    \begin{align*}
    Q_t \leq 8\wb{q}.
    \end{align*}
\end{theorem}
\textbf{Discussion}
Unlike in the batch setting, where the sampling procedure always returned
$m$ samples, the number of columns $Q_t$ selected by \inkoracle is a random variable, but with high probability it will be not much larger than $\wb{q}$.
Comparing \inkoracle to online kernel sparsification methods \cite{sun2012size},
we see that the number of columns, and therefore the space requirement,
is guaranteed to be small not only asymptotically but at each step, and
that no assumption on the spectrum of the matrix is required. Instead,
the space complexity naturally scales with the effective dimension of the problem,
and old samples that become superfluous are automatically discarded.
Comparing Thm.\,\ref{thm:ink-oracle-main} to Prop.\,\ref{prop:alaoui-nyst-guar}, \inkoracle achieves the same 
performance as its batch counterpart, as long as the space budget~$\wb{q}$ is
large enough. This budget depends on several quantities that are difficult
to estimate, such as the effective dimension of the full kernel matrix.
In practice, this quantity can be interpreted as the maximum amount of
space that the user can afford for the algorithm to run. If the
actual complexity of the problem exceeds this budget, the user can
choose to run it again with another parameter $\gamma$ or a worse accuracy
$\varepsilon$.
It is important to notice that, as we show in the proof, the distribution
induced by the sampling procedure of \inkoracle is not the same as the distribution
obtained by the multinomial sampling of \batchexact. Nonetheless, in our
analysis we show that the bias introduced by the different distribution
is small, and this allows \inkoracle to match the approximation guarantees given by
\citet{alaoui2014fast}.

We give a detailed proof of Thm.\,\ref{thm:ink-oracle-main} in App.~\ref{sec:app-concentration-proofs}. In the rest of this section
we sketch the proof and give the intuition for the most relevant parts.

The \shrinkop step uses the thresholding condition to guarantee that the
weight $b_{i,t}$ are good approximations of the $\wt{p}_{i,t}$.
To make the condition effective, we require that the
approximate probabilities $\wt{p}_{j,t}$ are decreasing.
Because the approximate probabilities follow the true probabilities $p_{i,t}$, we 
first show that this decrease happens for the exact case. 
\begin{lemma}\label{lem:monotone-decrease-prob}
    For any kernel matrix $\kermatrix_t$ at time $t$, and its bordering
    $\kermatrix_{t+1}$ at time $t+1$ we have that the probabilities $p_{i,t}$ are monotonically decreasing
    over time $t$,
    \begin{align*}
        \frac{\tau_{i,t+1}}{\deff{\gamma}_{t+1}} = p_{i,t+1}  \leq  p_{i,t} =  \frac{\tau_{i,t+1}}{\deff{\gamma}_t} \cdot
    \end{align*}
\end{lemma}
Since ridge leverage scores represent the importance of a column, when a new column arrives, there are two cases that can happen.
If the column is orthogonal to the existing
matrix, none of the previous leverage scores changes.
If the new column can explain part of the previous columns,
the previous columns should be picked less often, and we
expect $\tau_{i,t}$ to decrease.
Contrary to RLS, the effective dimension increases when the new sample
is orthogonal to the existing matrix, while it stays the same when
the new sample is a linear combination of the existing ones.
In addition, the presence of $\gamma$ regularizes both cases.
When the vector is nearly orthogonal, the presence of $\gamma\:I$ in the
inverse will still penalize it, while the $\gamma$ term at the denominator
of $\Delta$ will reduce the influence of linearly correlated samples.
Because $\tau_{i,t}$ decreases over time and $\deff{\gamma}_{i,t}$
increases, the probabilities $p_{i,t}$ will overall decrease over time.
This result itself is not sufficient to guarantee a well defined \shrinkop step.
Due to the $(\alpha,\beta)$-approximation, it is possible that 
$p_{i,t+1} \leq p_{i,t}$ but $\wt{p}_{i,t+1} \nleq \wt{p}_{i,t}$.
To exclude this possibility, we adapt the following idea from \citet{kelner2012spectral}.
\begin{proposition}[\citet{kelner2012spectral}]\label{prop:kl-take-minimum}
    Given the approximate probabilities $\wt{\:p}_t$ returned by an
    \texttt{$(\alpha,\beta)$-oracle} at time $t$, and the approximate probabilities
    $\wt{\:p}_{t+1}$ returned by an \texttt{$(\alpha,\beta)$-oracle} at time
    $\{t+1\}$, then the approximate probabilities
    $\min\!\footnote{element-wise mininum}\{\wt{\:p}_t, \wt{\:p}_{t+1}\}$ are also $(\alpha, \beta)$-approximate
    for $\{t+1\}$.
    Therefore, without loss of generality, we can assume that
    $\wt{p}_{i,t+1} \leq \wt{p}_{i,t}.$
\end{proposition}

Combining Lem.\,\ref{lem:monotone-decrease-prob} and Prop.\,\ref{prop:kl-take-minimum},
we can guarantee that at each step the $\wt{p}_{i,t}$-s decrease.
Unlike in the batch setting~\cite{alaoui2014fast},
we have to take additional care to consider correlations between iterations,
the fact that the inclusion probabilities of Alg.\,\ref{alg:ink-oracle-revamp}
are different from the multinomial ones of \directsample,
and that the number of columns kept at each iteration is a \emph{random} quantity $Q_{t}$. 
We adapt the approach of \citet{pachocki2016analysis} to the KRR setting to analyse this process.
The key aspect is that the reweighting and rejection rule
on line \ref{alg:decision-rule-oracle} of Alg.\,\ref{alg:ink-oracle-revamp}
will only happen when the probabilities are truly changing.
Finally, using a concentration inequality, we show that the number $Q_t$ of columns selected 
is with high probability only a constant factor away from the budget $\wb{q}$ given to the algorithm.

\section{LEVERAGE SCORES AND EFFECTIVE DIMENSION ESTIMATION}\label{sec:incremental-oracle}
In the previous section we showed that our incremental sampling strategy based on (estimated) RLSs has strong space and
approximation guarantees for $\akermatrix_n$.
While the analysis reported in the previous section relied on the existence of 
an \texttt{$(\alpha,\beta)$-oracle} returning accurate leverage scores and effective 
dimension estimates, in this section we show that such an oracle \emph{exists 
    and can be implemented efficiently.} This is obtained by \emph{two 
    separate estimators} for the RLSs and effective dimension that are updated 
incrementally and combined together to determine the sampling probabilities.
\subsection{Leverage Scores}\label{ss:ls}
We start by constructing an estimator that at each time $t$,
takes as input an approximate kernel matrix~$\akermatrix_t$, and returns
$\alpha$-approximate RLS $\atau_{i,t+1}$. The incremental nature of the estimator lies in
the fact that it exploits access to the columns already in $\selmatrix_t$
and the new (exact) column $\wb{\:k}_{t+1}$.
We give the following approximation guarantees.

\begin{lemma}\label{lem:fast-rls}
    We assume $\akermatrix_t$ satisfies Eq.~\ref{cond:alaoui-nyst-guar}
    and define~$\bkermatrix_{t+1}$ as the matrix bordered with the new row and column
        \begin{align*}
        \bkermatrix_{t+1} = \left[
    \begin{array}{c|c}
        \akermatrix_t & \wb{\:k}_{t+1} \\
    \hline
    \wb{\:k}_{t+1}^\transp & k_{t+1}
    \end{array}
    \right].
        \end{align*}
        Then
    \begin{align*}
        \:0 \preceq \kermatrix_{t+1} - \bkermatrix_{t+1} \preceq  \frac{\gamma}{1-\varepsilon}\:I.
    \end{align*}
    Moreover let $\alpha = \frac{2-\varepsilon}{1-\varepsilon}$ and
    \begin{align}\label{eq:rls-estimator}
        \atau_{i,t+1} =& \frac{1}{\alpha\gamma}\left(k_{i,i} - \:k_{i,t+1} \left(\bkermatrix_{t+1} + \alpha\gamma\:I\right)^{-1}\:k_{i,t+1}\right).
    \end{align}
    Then,
    for all $i$ such that $\:k_{i,t+1} \in \mathcal{I}_{t} \cup \{t+1\}$,
    \begin{align*}
        \frac{1}{\alpha}\tau_{i,t+1}(\gamma) \leq \atau_{i,t+1} \leq \tau_{i,t+1}(\gamma).
    \end{align*}
    \end{lemma}
\paragraph{Remark} 
 There are two important details that are used in proof of 
Lem.\,\ref{lem:fast-rls} (App.\,\ref{sec:app:incremental-oracle}). First, notice that
using $\akermatrix_t$ to approximate RLSs directly, would not be accurate
enough. RLSs are defined as
$\tau_{i,t}(\gamma) = \:e_i^\transp \kermatrix_t(\kermatrix_{t} + \gamma\:I)^{-1}\:e_i$
and while the product $(\kermatrix_{t} + \gamma\:I)^{-1}\:e_i$ can be
accurately reconstructed using $(\akermatrix_{t} + \gamma\:I)^{-1}\:e_i$, the
multiplication $\kermatrix_t\:e_i$ cannot be approximated well using
$\akermatrix_t$.  Since the nullspace of
$\akermatrix_{t}$ can be larger than the one of~$\kermatrix_t$, it is possible
that~$\:e_i$ partially falls into it, thus compromising the accuracy of the
approximation of the RLS.  In our approach, we deal with this problem by using the
\emph{actual columns} $\:k_{i,t}$ of $\kermatrix_t$ to compute the RLS.
This way, we preserve as
much as exact information of the matrix as possible, while the expensive
inversion operation is performed on the smaller approximation $\akermatrix_t$.
Since we require access to the stored columns $\:k_{i,t}$, our approach can
approximate the RLSs only for columns present in the dictionary but this is enough,
since we are only interested in accurate probabilities for columns in the dictionary and
for the new column $\wb{\:k}_{t+1}$ (which is available at time $t+1$).
As a comparison, the two-pass approach of \citet{alaoui2014fast} uses the
first pass just to compute an approximation $\akermatrix_n$, and then
approximates all leverage scores with $\akermatrix_n(\akermatrix_n + \gamma\:I)^{-1}$.
This has an impact on their approximation factor $\alpha$, that is proportional
to $(\lambda_{\min}(\kermatrix_n) - \gamma\varepsilon)$.
Therefore to have $\alpha \approx (\lambda_{\min}(\kermatrix_n) - \gamma\varepsilon) > 0$,
it is necessary that $\gamma\varepsilon$
is of the order of $\lambda_{\min}(\kermatrix_n)$, which in some cases can be very small,
and strongly increase the space requirements of the algorithm.
Using the actual columns of the matrix in Eq.\,\ref{eq:rls-estimator}
allows us to compute an $\alpha$-approximation independent of the smallest eigenvalue.
\subsection{Effective Dimension}\label{ss:effd}
Using Eq.~\ref{eq:rls-estimator}, we can estimate all the
RLSs that we need to update $\selmatrix_t$.
Nonetheless, to prove that the number of columns selected is not too large,
the proof of Thm.~\ref{thm:ink-oracle-main} in the appendix requires
that the sum of the probabilities $\wt{p}_{i,t}$ is smaller than 1.
Therefore we not only need to compute the RLSs, but also a normalization constant.
Indeed, a na\"{i}ve definition of the
probability $\wt{p}_{i,t}$ could be
    $p_{i,t} = \frac{\atau_{i,t}}{\sum_{j=1}^t\atau_{i,t}}\cdot$
A major challenge in our setting is that we cannot compute the sum of the 
approximate RLSs, because we do not have access to all the columns. 
Fortunately, we know that
    $\sum_{j=1}^t\atau_{i,t} \leq \sum_{j=1}^t\tau_{i,t}(\gamma) = \deff{\gamma}_t.$
Therefore, one of our technical contribution is an estimator $\adeff{\gamma}_t$
that does not use the approximate RLSs for the the columns that we no longer have.
We now define this estimator and state its approximation accuracy.
\begin{lemma}\label{lem:fast-deff}
    Assume $\akermatrix_t$ satisfies Eq.\,\ref{cond:alaoui-nyst-guar}.
    Let $\alpha = \left(\frac{2-\varepsilon}{1-\varepsilon}\right)$ and
    $\beta = \left(\frac{2-\varepsilon}{1-\varepsilon}\right)^{2}\left(1+\rho\right)$
    with  $\rho = \frac{\lambda_{\max}(\kermatrix_n)}{\gamma}$.
    Define
    \begin{align}
    \label{eq:deff-estimator}
        &\adeff{\gamma}_{t+1} =\adeff{\gamma}_t + \alpha\wt{\Delta}_t
    \end{align}
    with
    \begin{align}
         \wt{\Delta}_t =& \frac{1}{k_{t+1} +\gamma - \wb{\:k}_{t+1}^\transp \left(\akermatrix_{t} + \alpha\gamma\:I\right)^{-1}\wb{\:k}_{t+1}}\nonumber\\
         \times&\left( k_{t+1} -
                                \wb{\:k}_{t+1}^\transp \left(\akermatrix_{t} + \alpha\gamma\:I\right)^{-1}\wb{\:k}_{t+1} \right.\nonumber\\
                                        -&\left. \frac{(1-\varepsilon)^2}{4}\gamma\wb{\:k}_{t+1}^\transp (\akermatrix_{t} + \gamma\:I)^{-2}\wb{\:k}_{t+1}\right).
    \end{align}
    Then
    \begin{align*}
        \deff{\gamma}_{t+1} \leq \adeff{\gamma}_{t+1} \leq \beta\deff{\gamma}_{t+1}.
    \end{align*}
\end{lemma}
\paragraph{Discussion}
Since we cannot compute accurate RLSs for columns that are not present in the dictionary,
we prefer to not estimate how each RLSs changes over time, but instead we directly estimate
the increment of their sum.
We do it by updating our estimate $\adeff{\gamma}_{t+1}$
using our previous estimate $\adeff{\gamma}_{t}$, and $\wt{\Delta}_{t}$.
$\wt{\Delta}_t$ captures directly the interaction
of the new sample with the \textit{aggregate} of the previous samples, and
allows us to estimate the increase in effective dimension using only
the current matrix approximation $\akermatrix_t$, the new column $\wb{\:k}_{t+1}$
and the scalar $k_{t+1}$.
Differently from the other terms we studied, the numerator
of $\wt{\Delta}_{t}$ contains an additional
$\gamma\wb{\:k}_{t+1}^\transp (\akermatrix_{t} + \gamma\:I)^{-2}\wb{\:k}_{t+1}$
second order term.
The guarantees provided by Eq.\,\ref{cond:alaoui-nyst-guar} are not straightforward to extend
because in general if
$(\kermatrix_t + \gamma\:I)^{-1} \succeq (\akermatrix_{t} + \alpha\gamma\:I)^{-1}$,
it is not guaranteed that
$(\kermatrix_t + \gamma\:I)^{-2} \succeq (\akermatrix_{t} + \alpha\gamma\:I)^{-2}$.
Nonetheless, we show that $\akermatrix_t$ is still sufficient to estimate~$\wt{\Delta}_t$,
but, unlike $\alpha$, the approximation error $\beta$ is now dependent on the
spectrum.

\subsection{Analysis of \inkestimate}
With the separate estimates for leverage scores (Sect.\,\ref{ss:ls}) and
effective dimension (Sect.\,\ref{ss:effd}), we have the necessary ingredients 
for the \texttt{$(\alpha,\beta)$-oracle} and we are ready to present the final 
algorithm \inkestimate (Alg.~\ref{alg:ink-estimate-revamp}).

\begin{algorithm}[t!]
    \setstretch{1.2}
    \begin{algorithmic}[1]
        \renewcommand{\algorithmicrequire}{\textbf{Input:}}
        \renewcommand{\algorithmicensure}{\textbf{Output:}}
        \REQUIRE Dataset {$\dataset$}, regularization $\gamma$, sampling budget $\wb{q}$
        \ENSURE $\akermatrix_n$, $\selmatrix_n$
        \STATE Initialize $\mathcal{I}_0$ as empty, $\wt{p}_{1,0} = 1, b_{1,0} = 1$, budget $\wb{q}$
        \FOR{$t = 0,\dots,n-1$}
            \STATE Receive new column $\wb{\:k}_{t+1}$ and scalar $k_{t+1}$
            \STATE Compute $\alpha$-leverage scores $\{\atau_{i,t+1}: i \in \mathcal{I}_{t} \cup \{t+1\}\}$, using $\bkermatrix_{t+1}$, $\:k_i$, $k_{i,i}$, and Eq.\,\ref{eq:rls-estimator}
            \STATE Compute $\beta$-approximate $\adeff{\gamma}_{t+1}$ using $\akermatrix_{t}$, $\wb{\:k}_{t+1}$, $k_{t+1}$, and Eq.\,\ref{eq:deff-estimator}
            \STATE Set $\wt{p}_{i,t+1} = \min\{\atau_{i,t+1}/\adeff{\gamma}_{t+1}, \wt{p}_{i,t}\}$
            \STATE $\mathcal{I}_{{t+1}},\:b_{t+1}$ = $\shrinkexpandop(\mathcal{I}_{t}, \wt{\:p}_{t+1}, \:b_{t}, \wb{q})$
            \STATE Compute $\selmatrix_{t+1}$ using $\mathcal{I}_{t+1}$ and weights ${\sqrt{b_{i,t+1}}}$
            \STATE Compute $\akermatrix_{t+1}$ using $\selmatrix_{t+1}$ and Eq.\,\ref{eq:nystrom}
        \ENDFOR
        \STATE Return $\akermatrix_n$ and $\selmatrix_n$
    \end{algorithmic}
    \caption{The \inkestimate algorithm}
    \label{alg:ink-estimate-revamp}
\end{algorithm}

Using the approximation guarantees of Lem.\,\ref{lem:fast-rls} and Lem.\,\ref{lem:fast-deff}, we
are ready to state the final result, instantiating the generic $\alpha$ and
$\beta$ terms of Thm.~\ref{thm:ink-estimate-main} with the values obtained
in this section.
\begin{theorem}\label{thm:ink-estimate-main}
    Let $\rho = \lambda_{\max}(\kermatrix_t)/\gamma$, $\alpha = \left(\frac{2-\varepsilon}{1-\varepsilon}\right)$,
    $\beta = \left(\frac{2-\varepsilon}{1-\varepsilon}\right)^{2}(1+\rho)$,
and $\gamma > 1$. For any $0\leq \varepsilon \leq 1$,
and $0 \leq \delta \leq 1$, if we 
run Alg.\,\ref{alg:ink-estimate-revamp} with parameter $\wb{q}$, where
    \begin{align*}
        \wb{q} \geq\left( \frac{28 \alpha \beta\deff{\gamma}_t}{\varepsilon^2}\right)\log\left(\frac{4t}{\delta}\right),
    \end{align*}
to compute a sequence of random matrices $\selmatrix_t$ with a random number of columns $Q_t$,
then with probability $1-\delta$, for all $t$
the corresponding Nystr\"{o}m approximation $\akermatrix_t$ (Eq.~\ref{eq:nystrom}) satisfies \eqref{cond:alaoui-nyst-guar}
    \begin{align*}
        \:0 \preceq \kermatrix_t - \akermatrix_t \preceq  \frac{\gamma}{1-\varepsilon}\kermatrix_t(\kermatrix_t + \gamma\:I)^{-1} \preceq \frac{\gamma}{1-\varepsilon}\:I.
    \end{align*}
    With the same prob., \inkestimate requires at most
    \begin{align*}
        \bigotime(n^2 & \wb{q}^2 +  n\wb{q}^3)\\
        &\leq  \bigotime\left(\alpha^{2}\beta^{2}n^2\deff{\gamma}_{n}^{2} + \alpha^3\beta^3 n\deff{\gamma}_{n}^{3}\right)\\
        &=  \bigotime\left(\alpha^{4}(1+\rho)^{2}n^2\deff{\gamma}_{n}^{2} + \alpha^6(1 + \rho)^3 n\deff{\gamma}_{n}^{3}\right)
    \end{align*}
    time and the space is bounded as
    \begin{align*}
        \bigotime( n\wb{q}) \leq \bigotime\left(\alpha\beta n \deff{\gamma}_{n}\right) = \bigotime\left(\alpha^2(1+\rho) n \deff{\gamma}_{n}\right).
    \end{align*}
\end{theorem}
For the space complexity, from Thm.\,\ref{thm:ink-oracle-main} we know we will not select more than $\bigotime( \wb{q} )$ columns in high probability.
For the time complexity, at each iteration we need to solve linear systems
involving
$(\bkermatrix_{t+1} + \alpha\gamma\:I)^{-1}$ and
$(\akermatrix_{t} + \alpha\gamma\:I)^{-1}$. Approximating
the inverse using transformations similar to Eq.\,\ref{eq:linear-system-transformed}
takes $\bigotime(t\wb{q}^{2} + \wb{q}^{3})$ time, again using
a singular value decomposition approach. To compute all leverage scores,
we need to first compute an approximate inverse in $\bigotime(t\wb{q}^{2} + \wb{q}^{3})$
time, and then solve $Q_t$ systems, each using a multiplication costing
$\bigotime(tQ_t)$. With high probability, $Q_t \leq 8\wb{q}$, therefore
computing all leverage scores costs $\bigotime(t\wb{q}^{2} + \wb{q}^{3})$
for the first singular value decomposition, and $\bigotime(t\wb{q})$ for each
of the $\bigotime(\wb{q})$ applications.
To update the effective dimension estimate, we only have to compute another approximate inverse,
and that costs $\bigotime(t\wb{q}^{2} + \wb{q}^{3})$ as well.
Finally, we have to sum the costs over~$n$ steps, and from $\sum_{t=1}^n t\wb{q}^2 \leq \wb{q}^2n^2$,
we obtain the final complexity.
Even with a significantly different approach, \inkestimate achieves the same
approximation guarantees as \batchexact. Consequently, it provides the
same risk guarantees as the known batch version \cite{alaoui2014fast}, stated 
in the following corollary.
\begin{corollary}\label{cor:our-kern-reg-general}
    For every $t \in \{1,\dots,n\}$, let $\kermatrix_t$ be the kernel matrix
    at time $t$. Run Alg.\,\ref{alg:ink-estimate-revamp} with regularization
    parameter $\gamma$ and space budget $\wb{q}$.
    Then, at any time $t$, the solution
    $\wt{\:w}_t$ computed using the regularized \nystrom approximation
    $\akermatrix_t$ satisfies
    \begin{align*}
        \mathcal{R}(\wt{\:w}_t) \leq& \left(1 + \frac{\gamma}{\mu}\frac{1}{1-\varepsilon}\right)^{2} \mathcal{R}(\wh{\:w}_t)\\
        =& \left(1 + \frac{\lambda_{\max}(\kermatrix_t)}{\rho\mu}\frac{1}{1-\varepsilon}\right)^{2} \mathcal{R}(\wh{\:w}_t). 
    \end{align*}
\end{corollary}
\paragraph{Discussion}
Thm.\,\ref{thm:ink-estimate-main} combines the generic result of Thm.\,\ref{thm:ink-oracle-main}
with the actual implementation of an oracle that we developed in this section.
All the guarantees that hold for \inkoracle are inherited by \inkestimate,
but now we can quantify the impact of the errors $\alpha$ and $\beta$ on the algorithm.
As we saw, the $\alpha$ error does not depends on the time, the spectrum of
the kernel matrix or other quantities that increase over time.
On the other hand, estimating the effective dimension without having
access to all the leverage scores is a much harder task, and the $\beta$
factor depends on the spectrum through the $\rho$ coefficient. The influence that
this coefficient exerts on the space and time complexity can vary significantly as
the relative magnitude of $\lambda_{\max}(\kermatrix_n)$, $\gamma$ and $\mu$
changes. If the largest eigenvalue grows too large without a corresponding
increase in $\gamma$, the space and time requirements of \inkestimate can
grow, but the risk bound, depending on $\gamma/\mu$ remains small.
On the other hand, increasing $\gamma$ without increasing
$\mu$ reduces the computational complexity, but makes the guarantees
on the risk of the solution $\wt{\:w}_t$  much weaker.
As an example, \citet[Thm.~3]{alaoui2014fast} choose,
$\mu \geq \lambda_{\max}(\kermatrix_n)$ and $\gamma \approxeq \mu$.
If we do the same, we recover their bound.


\vspace{-0.2cm}
\section{CONCLUSION}\label{sec:conclusions}
\vspace{-0.2cm}
We presented a space-efficient algorithm
for sequential Nystr\"om approximation that requires only a single pass
over the dataset to construct a low-rank matrix $\akermatrix_n$
that accurately approximates the kernel matrix $\kermatrix_n$,
and compute an approximate KRR solution $\wt{\:w}_n$ whose risk
is close to the exact solution $\wh{\:w}_n$. All of these guarantees
do not hold only for the final matrix, but are valid for all intermediate
matrices $\akermatrix_t$ constructed by the sequential algorithm.

To address the challenges coming from the sequential setup, we introduced two separate estimators for RLSs and effective dimension
that provide multiplicative error approximations of these two quantities across iterations.
While the approximation of the RLSs is only a constant factor away from the
exact RLSs, the error of the approximate effective dimension scales
with the spectrum of the matrix through the coefficient $\rho$. A more careful analysis, or a different
estimator might improve this dependence, and they can be easily plugged to the
general analysis.

Our generalization results apply to the fixed design setting. An important
extension of our work would be to consider a random design, such as
in the work of \citet{rudi2015less}. This extension would
need even more careful tuning of the regularization parameter $\gamma$,
needing to satisfy requirements of both generalization
and the approximation of the \texttt{$(\alpha,\beta)$-oracle}.
Finally, the runtime analysis of the algorithm does not fully exploit the
sequential nature of the updates. An implementation based on decompositions
more amenable to updates (e.g., Cholesky decomposition), or on low-rank
solvers that can exploit hot-start might further improve the time complexity.

\paragraph{Acknowledgements}
\label{sec:Acknowledgements}
The research presented in this paper was supported by CPER Nord-Pas de Calais/FEDER DATA Advanced data science and technologies 2015-2020, French Ministry of
Higher Education and Research, Nord-Pas-de-Calais Regional Council,  French National Research Agency project ExTra-Learn (n.ANR-14-CE24-0010-01).

\bibliographystyle{plainnat}
\bibliography{library,library_daniele}

\begin{thebibliography}{24}
\providecommand{\natexlab}[1]{#1}
\providecommand{\url}[1]{\texttt{#1}}
\expandafter\ifx\csname urlstyle\endcsname\relax
  \providecommand{\doi}[1]{doi: #1}\else
  \providecommand{\doi}{doi: \begingroup \urlstyle{rm}\Url}\fi

\bibitem[Alaoui and Mahoney(2015)]{alaoui2014fast}
Ahmed~El Alaoui and Michael~W. Mahoney.
\newblock {Fast randomized kernel methods with statistical guarantees}.
\newblock In \emph{Neural Information Processing Systems (NeurIPS)}, 2015.

\bibitem[Bach(2013)]{bach_sharp_2012}
Francis Bach.
\newblock Sharp analysis of low-rank kernel matrix approximations.
\newblock In \emph{Conference on Learning Theory (COLT)}, 2013.

\bibitem[Drineas et~al.(2012)Drineas, Magdon-Ismail, Mahoney, and
  Woodruff]{drineas2011fast}
Petros Drineas, Malik Magdon-Ismail, Michael~W Mahoney, and David~P. Woodruff.
\newblock {Fast approximation of matrix coherence and statistical leverage}.
\newblock \emph{International Conference on Machine Learning (ICML)}, 2012.

\bibitem[Engel et~al.(2002)Engel, Mannor, and Meir]{engel2002sparse}
Yaakov Engel, Shie Mannor, and Ron Meir.
\newblock {Sparse online greedy support vector regression}.
\newblock In \emph{European Conference on Machine Learning (ECML)}, 2002.

\bibitem[Engel et~al.(2004)Engel, Mannor, and Meir]{engel2004kernel}
Yaakov Engel, Shie Mannor, and Ron Meir.
\newblock {The kernel recursive least-squares algorithm}.
\newblock \emph{IEEE Transactions on Signal Processing}, 52\penalty0
  (8):\penalty0 2275--2285, 2004.

\bibitem[Everitt(2002)]{everitt2002cambridge}
B.~S. Everitt.
\newblock \emph{The Cambridge dictionary of statistics}.
\newblock Cambridge University Press, Cambridge, 2002.

\bibitem[Gittens and Mahoney(2013)]{gittens_revisiting_2013}
Alex Gittens and Michael Mahoney.
\newblock Revisiting the {N}ystr\"{o}m method for improved large-scale machine
  learning.
\newblock In \emph{International Conference on Machine Learning (ICML)}, 2013.

\bibitem[Higham(2002)]{higham2002accuracy}
Nicholas~J Higham.
\newblock \emph{Accuracy and stability of numerical algorithms}.
\newblock Society for Industrial and Applied Mathematics, 2002.

\bibitem[Kelner and Levin(2012)]{kelner2012spectral}
Jonathan~A. Kelner and Alex Levin.
\newblock {Spectral sparsification in the semi-streaming setting}.
\newblock \emph{Theory of Computing Systems}, 53\penalty0 (2):\penalty0
  243--262, 2012.

\bibitem[Kivinen et~al.(2004)Kivinen, Smola, and Williamson]{kivinen2004online}
Jyrki Kivinen, Alexander~J. Smola, and Robert~C. Williamson.
\newblock {Online learning with kernels}.
\newblock \emph{IEEE Transactions on Signal Processing}, 52\penalty0
  (8):\penalty0 2165--2176, 2004.

\bibitem[Le et~al.(2013)Le, Sarl{\'{o}}s, and Smola]{le2013fastfood}
Quoc Le, Tam{\'{a}}s Sarl{\'{o}}s, and Alex~J Smola.
\newblock {Fastfood --- Approximating kernel expansions in loglinear time}.
\newblock In \emph{International Conference on Machine Learning (ICML)}, 2013.

\bibitem[Liu et~al.(2009)Liu, Park, and Principe]{liu_information_2009}
Weifeng Liu, Il~Park, and Jose~C. Principe.
\newblock An information theoretic approach of designing sparse kernel adaptive
  filters.
\newblock \emph{{IEEE} Transactions on Neural Networks}, 20\penalty0
  (12):\penalty0 1950--1961, 2009.

\bibitem[Pachocki(2016)]{pachocki2016analysis}
Jakub Pachocki.
\newblock Analysis of resparsification.
\newblock \emph{arXiv preprint arXiv:1605.08194}, 2016.

\bibitem[Rahimi and Recht(2007)]{rahimi2007random}
Ali Rahimi and Ben Recht.
\newblock {Random features for large-scale kernel machines}.
\newblock In \emph{Neural Information Processing Systems (NeurIPS)}, 2007.

\bibitem[Richard et~al.(2009)Richard, Bermudez, and Honeine]{richard_online_2009}
C{\'e}dric Richard, Jos{\'e} Carlos~M. Bermudez, and Paul Honeine.
\newblock Online prediction of time series data with kernels.
\newblock \emph{{IEEE} Transactions on Signal Processing}, 57\penalty0
  (3):\penalty0 1058--1067, 2009.

\bibitem[Rudi et~al.(2015)Rudi, Camoriano, and Rosasco]{rudi2015less}
Alessandro Rudi, Raffaello Camoriano, and Lorenzo Rosasco.
\newblock {Less is more: Nystr{\"{o}}m computational regularization}.
\newblock In \emph{Neural Information Processing Systems (NeurIPS)}, 2015.

\bibitem[Sch{\"{o}}lkopf and Smola(2001)]{scholkopf2001learning}
Bernhard Sch{\"{o}}lkopf and Alexander~J. Smola.
\newblock \emph{{Learning with kernels: Support vector machines,
  regularization, optimization, and beyond}}.
\newblock MIT Press, 2001.

\bibitem[Shawe-Taylor and Cristianini(2004)]{shawe2004kernel}
John Shawe-Taylor and Nelo Cristianini.
\newblock \emph{{Kernel methods for pattern analysis}}.
\newblock Cambridge University Press, 2004.

\bibitem[Sun et~al.(2012)Sun, Schmidhuber, and Gomez]{sun2012size}
Yi~Sun, J{\"{u}}rgen Schmidhuber, and Faustino~J. Gomez.
\newblock {On the size of the online kernel sparsification dictionary}.
\newblock In \emph{International Conference on Machine Learning (ICML)}, 2012.

\bibitem[Tropp(2011)]{tropp2011freedman}
Joel~A Tropp.
\newblock Freedman’s inequality for matrix martingales.
\newblock \emph{Electron. Commun. Probab}, 16:\penalty0 262--270, 2011.

\bibitem[Tropp(2015)]{tropp2015an-introduction}
Joel~A. Tropp.
\newblock An introduction to matrix concentration inequalities.
\newblock \emph{Foundations and Trends in Machine Learning}, 8\penalty0
  (1-2):\penalty0 1--230, 2015.

\bibitem[Van~Vaerenbergh et~al.(2012)Van~Vaerenbergh, L{\'a}zaro-Gredilla, and
  Santamar{\'\i}a]{van2012kernel}
Steven Van~Vaerenbergh, Miguel L{\'a}zaro-Gredilla, and Ignacio
  Santamar{\'\i}a.
\newblock Kernel recursive least-squares tracker for time-varying regression.
\newblock \emph{IEEE Transactions on Neural Networks and Learning Systems},
  23\penalty0 (8):\penalty0 1313--1326, 2012.

\bibitem[Williams and Seeger(2001)]{williams2001using}
Christopher Williams and Matthias Seeger.
\newblock {Using the Nystrom method to speed up kernel machines}.
\newblock In \emph{Neural Information Processing Systems (NeurIPS)}, 2001.

\bibitem[Zhang et~al.(2015)Zhang, Duchi, and Wainwright]{zhang2015divide}
Yuchen Zhang, John~C. Duchi, and Martin~J. Wainwright.
\newblock {Divide and conquer kernel ridge regression: A distributed algorithm
  with minimax optimal rates}.
\newblock \emph{Journal of Machine Learning Research}, 16:\penalty0 3299--3340,
  2015.

\end{thebibliography}
\newpage
\appendix
\onecolumn
\section{Extended proofs of Section \ref{sec:incremental-rejection}}\label{sec:app:incremental-rejection}

\begin{proof}[Proof of Lemma \ref{lem:monotone-decrease-prob}]

The proof follows from studying the evolution of the numerator and denominator (i.e., RLSs and effective dimension) separately.

\begin{lemma}
    \label{lem:monotone-increase-ridge}
    For any kernel matrix $\kermatrix_t$ and its bordering $\kermatrix_{t+1}$,
    we have for all $i = \{1,\dots,t\}$
    \begin{align*}
        \tau_{i,t}(\gamma) \geq \tau_{i,t+1}(\gamma)
    \end{align*}
\end{lemma}

\begin{proof}[Proof of Lemma \ref{lem:monotone-increase-ridge}]

We want to show that for all $i = \{1,\dots,t\}$
\begin{align*}
\frac{\tau_{i,t+1}(\gamma)}{\tau_{i,t}(\gamma)} = \frac{\:e_{i,t+1}^\transp \kermatrix_{t+1}(\kermatrix_{t+1}^\gamma)^{-1}\:e_{i,t+1}}{\:e_{i,t}^\transp \kermatrix_{t}(\kermatrix_{t}^\gamma)^{-1}\:e_{i,t}} \leq 1,
\end{align*}
where $\kermatrix_{t}^\gamma = \kermatrix_{t} + \gamma\:I_t$. 
Since $\kermatrix_{t+1}$ is obtained as the bordering of $\kermatrix_t$ using the vector $\wb{\:k}_{t+1} \in \Re^t$ and the element $k_{t+1}$ (see Eq.~\ref{eq:ker.bordering}), we can use the block matrix inverse formula and obtain
\begin{align*}
(\kermatrix_{t+1}^\gamma)^{-1}=        &\left[
    \begin{array}{c|c}
        (\kermatrix_t^\gamma) & \wb{\:k}_{t+1} \\
    \hline
    \wb{\:k}_{t+1}^\transp & k_{t+1} + \gamma
    \end{array}
\right]^{-1}
=\frac{1}{\xi}
  \left[
    \begin{array}{c|c}
        \begin{aligned}&\xi(\kermatrix_t^\gamma)^{-1}+ (\kermatrix_t^\gamma)^{-1}\wb{\:k}_{t+1} \wb{\:k}_{t+1}^\transp (\kermatrix_t^\gamma)^{-1}\end{aligned} & -(\kermatrix_t^\gamma)^{-1}\wb{\:k}_{t+1} \\
    \hline
    -\wb{\:k}_{t+1}^\transp (\kermatrix_t^\gamma)^{-1} & 1
    \end{array}
\right],
\end{align*}
with
\begin{align}
    \xi = k_{t+1} +\gamma - \wb{\:k}_{t+1}^\transp (\kermatrix_t^\gamma)^{-1} \wb{\:k}_{t+1}.
\end{align}
Then we can rewrite the ratio between the RLSs as
\begin{align*}
    \frac{\tau_{i,t+1}(\gamma)}{\tau_{i,t}(\gamma)} & =\frac{\:e_{i,t+1}^\transp \kermatrix_{t+1}(\kermatrix_{t+1}^\gamma)^{-1}\:e_{i,t+1}}{\:e_{i,t}^\transp \kermatrix_{t}(\kermatrix_{t}^\gamma)^{-1}\:e_{i,t}}\\
    &=\frac{[\:k_{i,t}^\transp, \enspace\kerfunc(x_{i}, x_{t+1})] \enspace\cdot\enspace[\xi(\kermatrix_{t}^\gamma)^{-1}\:e_{i,t}
    + (\kermatrix_{t}^\gamma)^{-1}\wb{\:k}_{t+1} \wb{\:k}_{t+1}^\transp (\kermatrix_{t}^\gamma)^{-1}\:e_{i,t}, \enspace-\wb{\:k}_{t+1}^\transp (\kermatrix_{t}^\gamma)^{-1}\:e_{i,t}]^\transp}{\xi\:k_{i,t}^\transp (\kermatrix_{t}^\gamma)^{-1}\:e_{i,t}}.
\end{align*}
where $\:k_{i,t} = \kermatrix_t \:e_{i,t}$. In the following we drop the dependency on $t$ from indicator vectors and kernel vectors and we write $\:k_{i} = \:k_{i,t}$ and $\:e_{i}=\:e_{i,t}$. As a result we can further rewrite the previous expression as
\begin{align*}
    \frac{\tau_{i,t+1}(\gamma)}{\tau_{i,t}(\gamma)} & = \frac{\xi\:k_{i}^\transp(\kermatrix_{t}^\gamma)^{-1} \:e_i + \:k_i^\transp (\kermatrix_{t}^\gamma)^{-1} \wb{\:k}_{t+1} \wb{\:k}_{t+1}^\transp (\kermatrix_{t}^\gamma)^{-1} \:e_i
    - \kerfunc(x_{i}, x_{t+1})\wb{\:k}_{t+1}^\transp(\kermatrix_{t}^\gamma)^{-1} \:e_i}{\xi \:k_{i}^\transp (\kermatrix_{t}^\gamma)^{-1}\:e_i}\\
    &=1 - \frac{\wb{\:k}_{t+1}^\transp (\kermatrix_{t}^\gamma)^{-1} \:e_i(\kerfunc(x_{i}, x_{t+1}) - \:k_i^\transp (\kermatrix_{t}^\gamma)^{-1} \wb{\:k}_{t+1})}{\xi \:k_{i}^\transp (\kermatrix_{t}^\gamma)^{-1}\:e_i}.
\end{align*}
We first focus on the term in parenthesis at the numerator. Since by definition $\kerfunc(x_i, x_{t+1}) = \:e_i^\transp\wb{\:k}_{t+1}$, we have
\begin{align*}
    &\kerfunc(x_{i}, x_{t+1}) - \:k_i^\transp (\kermatrix_{t}^\gamma)^{-1} \wb{\:k}_{t+1} =\:e_i^\transp\wb{\:k}_{t+1} - \:e_i^\transp\kermatrix_t(\kermatrix_{t}^\gamma)^{-1} \wb{\:k}_{t+1}\\
    &=\:e_i^\transp\wb{\:k}_{t+1} - \:e_i^\transp(\kermatrix_t + \gamma\:I)(\kermatrix_{t}^\gamma)^{-1} \wb{\:k}_{t+1} + \gamma\:e_i^\transp(\kermatrix_{t}^\gamma)^{-1} \wb{\:k}_{t+1}\\
    &=\:e_i^\transp\wb{\:k}_{t+1} - \:e_i^\transp(\kermatrix_{t}^\gamma)(\kermatrix_{t}^\gamma)^{-1} \wb{\:k}_{t+1} + \gamma\:e_i^\transp(\kermatrix_{t}^\gamma)^{-1} \wb{\:k}_{t+1} =\gamma\:e_i^\transp(\kermatrix_{t}^\gamma)^{-1} \wb{\:k}_{t+1}.
\end{align*}
Thus the ratio can be finally written as
\begin{align*}
    \frac{\tau_{i,t+1}(\gamma)}{\tau_{i,t}(\gamma)} & = 1 -  \frac{\gamma(\wb{\:k}_{t+1}^\transp (\kermatrix_{t}^\gamma)^{-1} \:e_i)^2}{\xi\:k_{i}^\transp (\kermatrix_{t}^\gamma)^{-1}\:e_i}.
\end{align*}
Since $\gamma > 0$, we only need to analyze the sign of the denominator to prove the final statement. Since $\:k_{i}^\transp (\kermatrix_{t}^\gamma)^{-1}\:e_i = \tau_{i,t} > 0$ by definition of RLS, the only term which needs some care is the coefficient $\xi$ of the block matrix inverse formula. 
As illustrated in Sect.\,\ref{sec:extact.krr}, the kernel function applied to any to points is $\kerfunc(x_i,x_{j}) = \langle \phi(x_i), \phi(x_{j})\rangle$, where $\varphi$ is the feature map. As a result, the kernel matrix and kernel vector can be written as $\kermatrix_t = \featkermatrix_t^\transp \featkermatrix_t$ and $\wb{\:k}_{t+1} = \featkermatrix_t^\transp\:\varphi(x_{t+1})$, where $\featkermatrix_t \in \Re^{D\times t}$. Let $\featkermatrix_t = \:V\:\Sigma\:U^\transp$ be the SVD decomposition of
$\featkermatrix$, with $\:V \in \Real^{D \times D}$ and $\:U \in \Real^{t \times t}$ are
orthonormal and $\:\Sigma \in \Real^{D \times t}$ is a rectangular diagonal matrix with singular values $\sigma_i$ on the diagonal.
We have
\begin{align*}
    \wb{\:k}_{t+1}^\transp (\kermatrix_t +\gamma\:I) \wb{\:k}_{t+1}
    &= \:\varphi(x_{t+1})^\transp\featkermatrix_t(\featkermatrix_t^\transp\featkermatrix_t + \gamma\:I)^{-1}\featkermatrix_t^\transp\:\varphi(x_{t+1})\\
    &= \:\varphi(x_{t+1})^\transp\:V\:\Sigma(\:\Sigma^\transp\:\Sigma + \gamma\:I)^{-1}\:\Sigma^\transp\:V^\transp\:\varphi(x_{t+1}).
\end{align*}
The central term $\:\Sigma(\:\Sigma^\transp\:\Sigma + \gamma\:I)^{-1}\:\Sigma^\transp$ is still a rectangular diagonal matrix with elements $\frac{\sigma_i^2}{\sigma_{i}^{2} + \gamma}$ for $i\in[t]$. With an abuse of notation we denote it as $\Diag\Big(\Big\{\frac{\sigma_i^2}{\sigma_{i}^{2} + \gamma}\Big\}_i\Big)$. Then we can write the previous expression as
\begin{align*}
    \wb{\:k}_{t+1}^\transp (\kermatrix_t +\gamma\:I) \wb{\:k}_{t+1}
    &= \:\varphi(x_{t+1})^\transp\:V \Diag\bigg(\Big\{\frac{\sigma_i^2}{\sigma_{i}^{2} + \gamma}\Big\}_i\bigg) \:V^\transp\:\varphi(x_{t+1})\\
    &\leq\:\varphi(x_{t+1})^\transp\:V\:I\:V^\transp\:\varphi(x_{t+1}) = \:\varphi(x_{t+1})^\transp\:\varphi(x_{t+1}) = \kerfunc(x_{t+1}, x_{t+1}) = k_{t+1},
\end{align*}
which implies that 
\begin{align*}
    &\xi = k_{t+1} +\gamma - \wb{\:k}_{t+1}^\transp (\kermatrix_{t}^\gamma)^{-1} \wb{\:k}_{t+1} \geq \gamma > 0,
\end{align*}
thus concluding the proof.

\end{proof}


\begin{lemma}
    \label{lem:monotone-increase-deff}
    For any kernel matrix $\kermatrix_t$ and its bordering $\kermatrix_{t+1}$,
    we have
\begin{align*}
    \deff{\gamma}_{t+1} = \deff{\gamma}_t + \Delta \geq \deff{\gamma}_t,
\end{align*}
where
\begin{align}
    &\Delta = \frac{k_{t+1} - \wb{\:k}_{t+1}^\transp (\kermatrix_{t}^\gamma)^{-1}\wb{\:k}_{t+1} - \gamma\wb{\:k}_{t+1}^\transp (\kermatrix_{t}^\gamma)^{-2}\wb{\:k}_{t+1}}{\xi} \geq 0\label{eq:deff-increase-exact}
\end{align}
and
\begin{align}
    &\xi = k_{t+1} +\gamma - \wb{\:k}_{t+1}^\transp (\kermatrix_{t}^\gamma)^{-1} \wb{\:k}_{t+1}\label{eq:shur-complement}.
\end{align}
\end{lemma}
\begin{proof}[Proof of Lemma \ref{lem:monotone-increase-deff}]
We first derive the incremental update of the effective dimension as
    \begin{align*}
        &\deff{\gamma}_{t+1} = \Tr\left(\kermatrix_{t+1}(\kermatrix_{t+1} + \gamma\:I)^{-1}\right)\\
        & =\frac{1}{\xi}\Tr \left( \left[
            \begin{array}{c|c}
                \kermatrix_t & \wb{\:k}_{t+1} \\
                \hline
                \wb{\:k}_{t+1}^\transp & k_{t+1}
            \end{array}
            \right]\left[
    \begin{array}{c|c}
        \xi(\kermatrix_{t}^\gamma)^{-1} + (\kermatrix_{t}^\gamma)^{-1}\wb{\:k}_{t+1} \wb{\:k}_{t+1}^\transp (\kermatrix_{t}^\gamma)^{-1} & -(\kermatrix_{t}^\gamma)^{-1}\wb{\:k}_{t+1} \\
    \hline
    -\wb{\:k}_{t+1}^\transp (\kermatrix_{t}^\gamma)^{-1} & 1
    \end{array}
\right]
        \right)\\
    &=\Tr(\kermatrix_{t}(\kermatrix_{t}^\gamma)^{-1})
    +\frac{\Tr(\kermatrix_{t}(\kermatrix_{t}^\gamma)^{-1}\wb{\:k}_{t+1} \wb{\:k}_{t+1}^\transp (\kermatrix_{t}^\gamma)^{-1})}{\xi}
    -2\frac{\wb{\:k}_{t+1}^\transp (\kermatrix_{t}^\gamma)^{-1}\wb{\:k}_{t+1}}{\xi} + \frac{k_{t+1}}{\xi}\\
    &\begin{aligned}
        = \deff{\gamma}_t &
        +\frac{\Tr((\kermatrix_{t} + \gamma\:I)(\kermatrix_{t}^\gamma)^{-1}\wb{\:k}_{t+1} \wb{\:k}_{t+1}^\transp (\kermatrix_{t}^\gamma)^{-1})}{\xi}\\
    &- \gamma\frac{\Tr((\kermatrix_{t}^\gamma)^{-1}\wb{\:k}_{t+1} \wb{\:k}_{t+1}^\transp (\kermatrix_{t}^\gamma)^{-1})}{\xi}
    -2\frac{\wb{\:k}_{t+1}^\transp (\kermatrix_{t}^\gamma)^{-1}\wb{\:k}_{t+1}}{\xi} + \frac{k_{t+1}}{\xi}\\
    = \deff{\gamma}_t &+\frac{\Tr(\:I\wb{\:k}_{t+1} \wb{\:k}_{t+1}^\transp (\kermatrix_{t}^\gamma)^{-1})}{\xi}
    - \gamma\frac{\Tr((\kermatrix_{t}^\gamma)^{-1}\wb{\:k}_{t+1} \wb{\:k}_{t+1}^\transp (\kermatrix_{t}^\gamma)^{-1})}{\xi}
    -2\frac{\wb{\:k}_{t+1}^\transp (\kermatrix_{t}^\gamma)^{-1}\wb{\:k}_{t+1}}{\xi} + \frac{k_{t+1}}{\xi}
    \end{aligned}\\
    &= \deff{\gamma}_t + \frac{k_{t+1} - \wb{\:k}_{t+1}^\transp (\kermatrix_{t}^\gamma)^{-1}\wb{\:k}_{t+1} - \gamma\wb{\:k}_{t+1}^\transp (\kermatrix_{t}^\gamma)^{-2}\wb{\:k}_{t+1}}{\xi},
\end{align*}
where we use the bordering of the kernel and the block matrix inversion formula. Using the same arguments as in Lemma~\ref{lem:monotone-increase-ridge}, we have $\xi>0$, thus we only need to focus on the numerator of the second term in the previous expression. We use the feature map matrix $\:\Phi_t$ and the fact that $\kermatrix_t = \:\Phi_t\:\Phi_t^\transp$ to obtain
\begin{align*}
&k_{t+1} - \wb{\:k}_{t+1}^\transp (\kermatrix_{t}^\gamma)^{-1}\wb{\:k}_{t+1} - \gamma\wb{\:k}_{t+1}^\transp (\kermatrix_{t}^\gamma)^{-2}\wb{\:k}_{t+1}\\
&=\:\varphi(x_{t+1})^\transp\:\varphi(x_{t+1}) - \:\varphi(x_{t+1})^\transp\featkermatrix_t(\featkermatrix_t \featkermatrix_t^\transp + \gamma)^{-1}\featkermatrix_t^\transp\:\varphi(x_{t+1}) - \gamma\:\varphi(x_{t+1})^\transp\featkermatrix_t(\featkermatrix_t \featkermatrix_t^\transp + \gamma)^{-2}\featkermatrix_t^\transp\:\varphi(x_{t+1})\\
&=\:\varphi(x_{t+1})^\transp\Big( \:I - \featkermatrix_t(\featkermatrix_t \featkermatrix_t^\transp + \gamma)^{-1}\featkermatrix_t^\transp - \gamma \featkermatrix_t(\featkermatrix_t \featkermatrix_t^\transp + \gamma)^{-2}\featkermatrix_t^\transp\Big)  \:\varphi(x_{t+1}).
\end{align*}
Using the singular value decomposition of the feature matrix as $\featkermatrix_t = \:V\:\Sigma\:U^\transp$, we can rewrite the central part of the quadratic form above as
\begin{align*}
&k_{t+1} - \wb{\:k}_{t+1}^\transp (\kermatrix_{t}^\gamma)^{-1}\wb{\:k}_{t+1} - \gamma\wb{\:k}_{t+1}^\transp (\kermatrix_{t}^\gamma)^{-2}\wb{\:k}_{t+1}\\
&=\:\phi(x_{t+1})^\transp\:V\Diag\left(\Big\{1-\frac{\sigma_i^2}{\sigma_i^2 + \gamma} - \frac{\gamma\sigma_i^2}{(\sigma_i^2 + \gamma)^2}\Big\}_{i=1}^t\right)\:V^\transp\:\phi(x_{t+1})\\
&=\:\varphi(x_{t+1})^\transp\:V\Diag\left( \bigg\{ \Big(\frac{\gamma}{\sigma_i^2+\gamma}\Big)^2 \bigg\}_{i=1}^d \right)\:V^\transp\:\varphi(x_{t+1}) > 0,
\end{align*}
which guarantees that the increment to the effective dimension is non-negative and concludes the proof.
\end{proof}
Combining the two lemmas, we obtain Lemma \ref{lem:monotone-decrease-prob}.
\end{proof}

\begin{proof}[Proof of Proposition \ref{prop:kl-take-minimum}]
    With a slight abuse of notation, for this proof we indicate with
    $\:p_{t+1} \leq \:p_t$ that for each $i \in \{1,\dots,t\}$ we have $p_{i,t+1} \leq p_{i,t}$.
    We know that $\wt{\:p}_t \leq \:p_t$ and $\wt{\:p}_{t+1} \leq \:p_{t+1}$. Then obviously
    \begin{align*}
        \min\{\wt{\:p}_t, \wt{\:p}_{t+1}\} \leq \min\{\:p_t, \:p_{t+1}\} \leq \:p_{t+1}
    \end{align*}
    For the lower bound, we know from Lemma \ref{lem:monotone-decrease-prob} that $\:p_{t+1} \leq \:p_t$.
    Therefore
    \begin{align*}
        \min\{\wt{\:p}_t, \wt{\:p}_{t+1}\} \geq \frac{1}{\alpha\beta}\min\{\:p_t, \:p_{t+1}\} \geq \frac{1}{\alpha\beta}\:p_{t+1}.
    \end{align*}
\end{proof}


\section{Proof of Theorem \ref{thm:ink-oracle-main} and \ref{thm:ink-estimate-main}}\label{sec:app-concentration-proofs}
\begin{proof}[Proof of Theorem \ref{thm:ink-oracle-main} and \ref{thm:ink-estimate-main}]
We will first prove the result for generic $(\alpha,\beta)$ approximate
probabilities, thus proving Theorem \ref{thm:ink-oracle-main}.
Theorem \ref{thm:ink-estimate-main} directly follows by placing the correct
values for $\alpha$ and $\beta$ computed in Section \ref{sec:incremental-oracle}.
We will first prove the result for generic $(\alpha,\beta)$ approximate
probabilities, thus proving Theorem \ref{thm:ink-oracle-main}.
Theorem \ref{thm:ink-estimate-main} directly follows by placing the correct
values for $\alpha$ and $\beta$ computed in Section \ref{sec:incremental-oracle}.
 For the proof, we will use the following matrix concentration inequalities.
 \begin{proposition}[Thm.~1.2~\cite{tropp2011freedman}]\label{prop:matrix-freedman}
Consider a matrix martingale $\{ \:Y_k : k = 0, 1, 2, \dots \}$ whose values are self-adjoint matrices with dimension $d$, and let $\{ \:X_k : k = 1, 2, 3, \dots \}$ be the difference sequence.  Assume that the difference sequence is uniformly bounded in the sense that
\begin{align*}
\lambda_{\max}( \:X_k ) \leq R
\quad\text{almost surely}
\quad\text{for $k = 1, 2, 3, \dots$}.
\end{align*}
Define the predictable quadratic variation process of the martingale:
\begin{align*}
\:{W}_k := \sum_{j=1}^k \expectedvalue \left[ \:X_j^2 \condbar \{\:X_{s}\}_{s=0}^{j-1} \right].
\quad\text{for $k = 1, 2, 3, \dots$}.
\end{align*}
Then, for all $t \geq 0$ and $\sigma^2 > 0$,
\begin{align*}
\probability\left( \exists k \geq 0 : \lambda_{\max}(\:Y_k) \geq t \ \text{ and }\ 
        \normsmall{\:W_k} \leq \sigma^2 \right)
	\leq d \cdot \exp \left\{ - \frac{ t^2/2 }{\sigma^2 + Rt/3} \right\}.
\end{align*}
\end{proposition}
 \begin{proposition}[Thm.~5.1.1~\cite{tropp2015an-introduction}]\label{prop:matrix-chernoff}
 Consider a finite sequence $\{\:V_{j}\}$ of independent random symmetric PSD matrices with 
 dimension $t$. Assume that $\lambda_{\max}(\:V_{j})\leq R$ and
 define
 \begin{align*}
 \mu = \left\Vert  \sum\nolimits_{j} \expectedvalue\left[\:V_j \right]\right\Vert_{2}.
 \end{align*}
then, for all $\delta \geq 0$,
\begin{align*}
\probability\left( \lambda_{\max}\left(\sum\nolimits_{j}\:V_j\right) \geq (1+\delta)\mu \right)
	\leq t\left( \frac{e^\delta}{(1+\delta)^{(1+\delta)}} \right)^{\mu/R}
\end{align*}
\end{proposition}

Let $\kermatrix_t = \:U\:\Lambda\:U^\transp$ be the eigendecomposition of the kernel matrix
at time $t$. Then we define $\:\Psi_t = \:\Lambda^{1/2}(\:\Lambda + \gamma\:I)^{-1/2}\:U^\transp$.
In the following we drop the dependency on $t$ and we use $\:\Psi_t = \:\Psi$ and we denote by $\:\psi_i$ the $i$-th
column of $\:\Psi$ so that
$$
\:\Psi\:\Psi^\transp=\sum_{i=1}^{t}\:{\psi}_{i}\:\psi_{i}^\transp.
$$
We recall that $\:\Psi$, which
describes the ratio between the spectrum of the kernel and its
soft-thresholded version, is strictly related to both the RLSs and
the effective dimension, since $\tau_{i,t}(\gamma) = \normsmall{[\:\Psi_t]_{:,i}}^{2}_{2}$ and
$\deff{\gamma}_t = \normsmall{\:\Psi_t^{2}}_{F}$.

We want to show that, for an appropriately chosen $\selmatrix_t$,
the largest eigenvalue of
\begin{align*}
\:\Psi\:\Psi^\transp - \:\Psi\selmatrix_t\selmatrix_t^\transp\:\Psi^\transp
=\sum_{i=1}^{t}\left(1-b_{i,t}\right)\:\psi_{i}\:\psi_{i}^\transp
\end{align*}
is small. In particular, \citet{alaoui2014fast} showed that if
$\lambda_{\max}(\:\Psi\:\Psi^\transp - \:\Psi\selmatrix_t\selmatrix_t^\transp\:\Psi^\transp) \leq \varepsilon$,
then the approximated matrix $\akermatrix_t$ deterministically satisfies
 \eqref{cond:alaoui-nyst-guar}.

\paragraph{Approximation accuracy}
The following proof follows closely the approach introduced in
\citet{pachocki2016analysis}. In particular, we will follow the random
evolution of the weights $b_{i,s}$ (defined by the behaviour
of Subroutine $\shrinkexpandop$) as the algorithm runs.
The subscript $(i,s)$ indexes these weights temporally, following the iterations
of our algorithm, but the relationship between $s$ and the value of $b_{i,s}$
is not immediate. In particular, if $\wt{p}_{i,s}$ does not decrease sufficiently
at one iteration, $b_{i,s} = b_{i,s-1}$ and the weight of that column stays
unchanged. If instead it decreases significantly, the inner loop
of \shrinkexpandop (lines \ref{alg:decision-rule-oracle} and following) can execute multiple times, and $b_{i,s} - b_{i,s-1}$ can
be greater than 1.
Nonetheless, if $b_{i,s} = l$, we know that none of the Bernoulli trials
    $\mathcal{B}\left(\frac{k}{k+1}\right)$ in the sequence
$k = \{1,2,\dots,l\}$ have failed, because if any of those trials failed
we would have set the variable to 0 forever.

We can represent both \inkoracle and \inkestimate as the following random
process
\begin{align*}
    \wh{\:Y}_{i,s} = 
    \left(\sum_{k=1}^{i}\left(1-b_{k,s}\right)\:\psi_{k}\:\psi_{k}^\transp\right)
    +\left(\sum_{k=i+1}^{t}\left(1-b_{k,s-1}\right)\:\psi_{k}\:\psi_{k}^\transp\right)
\end{align*}
where the differences $\wh{\:X}_{i,s} = \wh{\:Y}_{i,s} -\wh{\:Y}_{i,s-1}$ are
\begin{align*}
    &\wh{\:X}_{i,s} = 
\left(b_{i,s-1} - b_{i,s}\right)\:\psi_{i}\:\psi_{i}^\transp,
\end{align*}

For a consistent notation, we simply set $b_{i,s} = 1$
for all $s < i$.
Intuitively, for $t$ time steps, indexed by $s$, the algorithm loops over $t$ columns,
indexed by $i$. For each columns, it uses a deterministic function
$f(\{b_{i,s-1}\}_{i=1}^t)$ to compute $\wt{p}_{i,s}$.
If $\wt{p}_{i,s}$ decreases enough, compared to $\wt{p}_{i,s-1}$,
it executes one or more independent Bernoulli trials to either increase
or set the weight to 0.
In practice, the algorithm only loops
over the columns currently stored and the newly arrived column,
but the analysis implicitly considers the columns it dropped and has not seen yet.
At the end of a full iteration, the columns are actually discarded.
For the end of an iteration ($i = t$) we use the
shortened notation $\wh{\:Y}_{s} = \wh{\:Y}_{t,s}$.

To bound $\lambda_{\max}(\wh{\:Y}_{t})$, we will use
Freedman's inequality, in particular its
extensions for matrix martingales \cite{tropp2011freedman}
from Proposition \ref{prop:matrix-freedman}.
Instead of working directly on $\wh{\:Y}_{t}$, we'll consider a different
process that has access to information unaccessible to a realistic
algorithm but can be analyzed more easily. Consider the process
\begin{align*}
\:Y_{i,s} = 
    \left(
    \left(\sum_{k=1}^{i}\left(1-b_{k,s}\right)\:\psi_{k}\:\psi_{k}^\transp\right)
    +\left(\sum_{k=i+1}^{t}\left(1-b_{k,s-1}\right)\:\psi_{k}\:\psi_{k}^\transp\right)
    \right)\indfunc\left\{\normsmall{\:Y_{s-1}} \leq \varepsilon\right\}
    +\:Y_{s-1}\indfunc\left\{\normsmall{\:Y_{s-1}} \geq \varepsilon\right\}.
\end{align*}
This sequence represents a variant of our algorithm that can detect if the
previous iteration failed to construct an accurate approximation. When
this failure happens, it stops the process and continues until the end
without updating anything. It is clear to see that if any of the
intermediate elements of the process violates the condition, the last
element will violate it too. For the rest, $\:Y_{s}$ behaves exactly
like $\wh{\:Y}_{s}$. Therefore, we can write
\begin{align*}
\probability\left( \lambda_{\max}(\wh{\:Y}_{t}) \geq \varepsilon\right)
\leq \probability\left( \lambda_{\max}(\:Y_{t}) \geq \varepsilon\right).
\end{align*}
If we denote with $\:W_t$ the predictable quadratic variation of the process
$\:Y_t$, we can write
\begin{align*}
&\probability\left( \lambda_{\max}(\wh{\:Y}_{t}) \geq \varepsilon\right)
\leq \probability\left( \lambda_{\max}(\:Y_{t}) \geq \varepsilon\right)\\
&= \probability\left( \lambda_{\max}(\:Y_{t}) \geq \varepsilon \cap \lambda_{\max}(\:W_{t}) \leq \sigma^2\right)
+ \probability\left( \lambda_{\max}(\:Y_{t}) \geq \varepsilon \cap \lambda_{\max}(\:W_{t}) > \sigma^2\right)\\
&\leq \probability\left( \lambda_{\max}(\:Y_{t}) \geq \varepsilon \cap \lambda_{\max}(\:W_{t}) \leq \sigma^2\right)
+ \probability\left( \lambda_{\max}(\:W_{t}) > \sigma^2\right)
\end{align*}
Where the last inequality derives from the definition of probability of an
intersection of two events.
For the first term, we will use Proposition \ref{prop:matrix-freedman},
while for the second we can use Proposition \ref{prop:matrix-chernoff}.
We will now show that our processes satisfy the conditions required to
apply the propositions, and how to properly choose $\sigma^2$.

First, let's analyze the process $\:Y_{t}$.
We define the martingale difference as
\begin{align*}
\:X_{i,s} =&
\left(b_{i,s-1} - b_{i,s}\right)\:\psi_{i}\:\psi_{i}^\transp \indfunc\left\{\normsmall{\:Y_{s-1}} \leq \varepsilon \right\}
        +\:0\indfunc\left\{\normsmall{\:Y_{s-1}} \geq \varepsilon \right\}\\
=&\left(b_{i,s-1} - b_{i,s}\right)\:\psi_{i}\:\psi_{i}^\transp \indfunc\left\{\normsmall{\:Y_{s-1}} \leq \varepsilon \right\}
\end{align*}
First, we show that the martingale properties are satisfied. We begin by remarking
that conditioned on all events up to time $s-1$, $\wt{p}_{i,s}$ is
a fixed quantity. The only randomness across the iteration is the set of
random variables $\{b_{i,s}\}$, which (again, conditioned on the previous iteration)
are independent.
Moreover, if the indicator function is not active,
$\:Y_{s}$ is deterministically equal to
$\:Y_{s-1}$ (the process is stopped), and the martingale requirement
is satisfied. Therefore, for the reminder of the proof, we can safely assume
$\normsmall{\:Y_{s-1}} \leq \varepsilon$.
Define all the past filtration as
$\mathcal{F}_{i,s} = \{\{\:X_{k,s}\}_{k=0}^{i-1}, \:Y_{s-1} \}$.
Conditioned on $\mathcal{F}_{i,s}$, $b_{i,s-1}$ and $\wt{p}_{i,s}$ are
constants.
In this case,
\begin{align*}
&\expectedvalue_{b_{i,s}}\left[\:X_{i,s} \condbar \mathcal{F}_{i,s}\right]
=\expectedvalue_{b_{i,s}}\left[
    \left(b_{i,s-1} - b_{i,s}\right)\:\psi_{i}\:\psi_{i}^\transp \condbar \mathcal{F}_{i,s}\right]
=
    \left(b_{i,s-1} - \expectedvalue_{b_{i,s}}\left[b_{i,s}\condbar \mathcal{F}_{i,s}\right]\right)\:\psi_{i}\:\psi_{i}^\transp
\end{align*}
We should now compute the expected value of $b_{i,s}$. Without loss
of generality, assume $b_{i,s-1} = l$. Therefore, the algorithm will continue
to execute Bernoulli trials until $b_{i,s} \geq 1/(\wt{p}_{i,s}\wb{q})$. Notice
that given $\:Y_{s-1}$, this is a fixed quantity $l' > l$, that we will take as
target. If any of the $l' - l$ trials fail, the variable will be set to 0.
Therefore, its expected value is
\begin{align*}
&\expectedvalue\left[{b_{i,s}}\condbar \mathcal{F}_{i,s}\right]
= l' \probability\left(\mathcal{B}\left(\frac{l'-1}{l'}\right) = 1 \cap \mathcal{B}\left(\frac{l'-2}{l'-1}\right) = 1 \cap \dots \cap\mathcal{B}\left(\frac{l}{l+1}\right) = 1\right)\\
&= l' \probability\left(\mathcal{B}\left(\frac{l'-1}{l'}\right) = 1 \right)\probability\left( \mathcal{B}\left(\frac{l'-2}{l'-1}\right) = 1 \right)\probability\left( \dots\right) \probability\left(\mathcal{B}\left(\frac{l}{l+1}\right) = 1\right)\\
&= l' \frac{l'-1}{l'} \frac{l'-2}{l'-1} \cdots\frac{l}{l+1}
= l = b_{i,s-1}
\end{align*}
Then
\begin{align*}
&\expectedvalue_{b_{i,s}}\left[
    \left(b_{i,s-1} - b_{i,s}\right)\:\psi_{i}\:\psi_{i}^\transp \condbar \mathcal{F}_{i,s}\right]
= \left(b_{i,s-1} - b_{i,s}\frac{b_{i,s-1}}{b_{i,s}}
\right)\:\psi_{i}\:\psi_{i}^\transp
= 0
\end{align*}
This proves that our process is indeed a martingale.
We can now continue with the properties necessary to apply Proposition~\ref{prop:matrix-freedman}.
First we need to compute $R$ in Proposition \ref{prop:matrix-freedman}.
In our modified process, all the way up to time $t-1$, we are guaranteed
to have $\normsmall{\:Y_{s-1}} \leq \varepsilon$ (or the process is stopped).
Under this condition, \citet[App. A, Lemma 1]{alaoui2014fast} show that it holds deterministically that
\begin{align*}
        \frac{\tau_{i,s}}{\alpha\beta\deff{\gamma}_s} = \frac{p_{i,s}}{\alpha\beta} \leq \wt{p}_{i,s} \leq p_{i,s} = \frac{\tau_{i,s}}{\deff{\gamma}_s}
\end{align*}
Therefore, we know that
\begin{align*}
        \frac{\tau_{i,t}}{\alpha\beta\deff{\gamma}_t} = \frac{p_{i,t}}{\alpha\beta} \leq \wt{p}_{i,t} \leq \wt{p}_{i,s}
\end{align*}
From Subroutine \shrinkexpandop, we know that the condition to try to increase
$b_{i,s}$ is that $b_{i,s}\wt{p}_{i,s} \leq 1/\wb{q}$, let $M_i$ the smallest integer such that
\begin{align*}
2\frac{\alpha\beta}{p_{i,t}\wb{q}} \geq \frac{\alpha\beta}{p_{i,t}\wb{q}} + 1 \geq M_{i} \geq \frac{\alpha\beta}{p_{i,t}\wb{q}} \geq \frac{1}{\wt{p}_{i,t}\wb{q}}.
\end{align*}
from the condition, we know that the Algorithm will never try to increase
$b_{i,s}$ above $M_i$, and that if a column reaches this quantity, we will
surely keep it in the dictionary forever.
Therefore
\begin{align*}
\lambda_{\max}(\:X_{i,s})
&=\lambda_{\max}\left(\left(b_{i,s-1} - b_{i,s}\right)\:\psi_{i}\:\psi_{i}^\transp\right)
\leq \lambda_{\max}\left( b_{i,s-1} \:\psi_{i}\:\psi_{i}^\transp\right)\\
&\leq \lambda_{\max}\left( M_i \:\psi_{i}\:\psi_{i}^\transp\right)
\leq \lambda_{\max}\left( 2\frac{\alpha\beta}{p_{i,t}\wb{q}} \:\psi_{i}\:\psi_{i}^\transp\right)
=\lambda_{\max}\left( \frac{2\alpha\beta\deff{\gamma}_t}{\tau_{i,t}\wb{q}} \:\psi_{i}\:\psi_{i}^\transp\right)\\
&\leq\lambda_{\max}\left(\ \frac{2\alpha\beta\deff{\gamma}_t}{\wb{q}} \:I\right)
=\lambda_{\max}\left( \frac{2\alpha\beta\deff{\gamma}_t}{\wb{q}} \:I\right)
\leq 2\alpha\beta\deff{\gamma}_t/\wb{q} = R
\end{align*}
We need now to analyze $\:W_t$. Again, define all the past filtration as
$\mathcal{F}_{i,s} = \{\{\:X_{k,s}\}_{k=0}^{i-1}, \:Y_{s-1} \}$.
We have
\begin{align*}
\:W_{t} = \sum_{s=1}^{t}\sum_{i=1}^t \expectedvalue \left[ \:X_{i,s}^2 \condbar \mathcal{F}_{i,s} \right]
=\sum_{s=1}^{t}\sum_{i=1}^t \expectedvalue \left[ \left(b_{i,s-1} - b_{i,s}\right)^{2}\:\psi_{i}\:\psi_{i}^\transp\:\psi_{i}\:\psi_{i}^\transp  \condbar \mathcal{F}_{i,s} \right]
\end{align*}
Again, without loss of generality, assume $b_{i,s} = l'$, $b_{i,s-1} = l$.
We compute
\begin{align*}
&\sum_{s=1}^{t}\expectedvalue \left[ \left(b_{i,s-1} - b_{i,s}\right)^{2} \condbar \mathcal{F}_{i,s} \right]
=\sum_{s=1}^{t}b_{i,s-1}^2 - 2b_{i,s-1}\expectedvalue \left[ b_{i,s} \condbar \mathcal{F}_{i,s} \right] + \expectedvalue \left[ b_{i,s}^{2} \condbar \mathcal{F}_{i,s} \right]\\
&=\sum_{s=1}^{t}b_{i,s-1}^2 - 2b_{i,s-1}^2 + \expectedvalue \left[ b_{i,s}^{2} \condbar \mathcal{F}_{i,s} \right]
=\sum_{s=1}^{t} \expectedvalue \left[ b_{i,s}^{2} \condbar \mathcal{F}_{i,s} \right] - b_{i,s-1}^2
\end{align*}
Now,
\begin{align*}
&\expectedvalue\left[{b_{i,s}^2}\condbar \mathcal{F}_{i,s}\right]
= {l'}^2 \probability\left(\mathcal{B}\left(\frac{l'-1}{l'}\right) = 1 \cap \mathcal{B}\left(\frac{l'-2}{l'-1}\right) = 1 \cap \dots \cap\mathcal{B}\left(\frac{l}{l+1}\right) = 1\right)\\
&= {l'}^2 \probability\left(\mathcal{B}\left(\frac{l'-1}{l'}\right) = 1 \right)\probability\left( \mathcal{B}\left(\frac{l'-2}{l'-1}\right) = 1 \right)\probability\left( \dots\right) \probability\left(\mathcal{B}\left(\frac{l}{l+1}\right) = 1\right)\\
&= {l'}^2 \frac{l'-1}{l'} \frac{l'-2}{l'-1} \dots\frac{l}{l+1} = l'l = b_{i,s}b_{i,s-1}
\end{align*}
Let $B_i = \max_{s=1}^t b_{i,s}$ be the maximum achieved by $b_{i,s}$ before
failing a Bernoulli trial and being set to zero forever, and let $d_i = \argmax_{s=1}^t b_{i,s}$ be the
timestep where this is achieved. We can see
\begin{align*}
\:W_{t} 
&=\sum_{i=1}^t \sum_{s=1}^{t}b_{i,s-1}(b_{i,s} - b_{i,s-1})\:\psi_{i}\:\psi_{i}^\transp\:\psi_{i}\:\psi_{i}^\transp
=\sum_{i=1}^t \sum_{s=1}^{d_i}b_{i,s-1}(b_{i,s} - b_{i,s-1})\:\psi_{i}\:\psi_{i}^\transp\:\psi_{i}\:\psi_{i}^\transp\\
&=\sum_{i=1}^t \left(b_{i,d_i-1}b_{i,d_i} - \sum_{s=1}^{d_i-1}b_{i,s}(b_{i,s} - b_{i,s-1}) - b_{i,0}\right)\:\psi_{i}\:\psi_{i}^\transp\:\psi_{i}\:\psi_{i}^\transp\\
&\preceq\sum_{i=1}^t b_{i,d_i}b_{i,d_i} \:\psi_{i}\:\psi_{i}^\transp\:\psi_{i}\:\psi_{i}^\transp
=\sum_{i=1}^t B_i^2 \:\psi_{i}\:\psi_{i}^\transp\:\psi_{i}\:\psi_{i}^\transp
\end{align*}
Where we eliminated the inner negative summation because each of its elements
$b_{i,s}(b_{i,s} - b_{i,s-1})$ is surely positive or zero, and
the overall summation is negative.
We now have all the elements to apply the Freedman inequality, but we do
not know which value of $\sigma^2$ will hold in high probability.
Therefore, we will apply the Chernoff inequality to $\:W_t$ itself.
To begin, we note that the maximum number of trials that the algorithm
will carry out on $b_{i,s}$ is bounded by $\frac{1}{\wt{p}_{i,t}\wb{q}} + 1$.
Because $\wt{p}_{i,t}$ is not independent from $\wt{p}_{j,t}$, we have also
that $B_i$ is not independent from $B_j$. To simplify the analysis, we
will consider a process that continue trying to increase $b_{i,s}$ until
it reaches $M_i$. Define $B'_i$ as the maximum achieved in this augmented process.
Clearly, $B'_i \geq B_i$, because we only give $b_{i,s}$ more possibilities
to increase. But since $M_i$ is a fixed quantity, $B'_i$ is independent of $B'_j$.
Therefore we can write
\begin{align*}
\:W_{t}
\preceq\sum_{i=1}^t B_{i}^2\:\psi_{i}\:\psi_{i}^\transp\:\psi_{i}\:\psi_{i}^\transp
\preceq\sum_{i=1}^t {B'_{i}}^2\:\psi_{i}\:\psi_{i}^\transp\:\psi_{i}\:\psi_{i}^\transp
=\sum_{i=1}^t \:V_{i}
\end{align*}
and analyze the sum of independent of matrices $\:V_{i}$.
Again,
we know that the algorithm will never try to increase
$b_{i,s}$ above $M_i$, and therefore $B_i \leq M_i$.
We have
\begin{align*}
\lambda_{\max}(\:V_{i})
&=\lambda_{\max}\left({B_{i}'}^2\:\psi_{i}\:\psi_{i}^\transp\:\psi_{i}\:\psi_{i}^\transp\right)
\leq\lambda_{\max}\left(M_i^2\:\psi_{i}\:\psi_{i}^\transp\:\psi_{i}\:\psi_{i}^\transp\right)\\
&\leq \lambda_{\max}\left( \left(\frac{2\alpha\beta\deff{\gamma}_t}{\tau_{i,t}\wb{q}}\right)^2 \:\psi_{i}\:\psi_{i}^\transp\:\psi_{i}\:\psi_{i}^\transp\right)\\
&\leq\lambda_{\max}\left(\left(\frac{2\alpha\beta\deff{\gamma}_t}{\wb{q}}\right)^2 \:I\right)
\leq 4\left(\frac{\alpha\beta\deff{\gamma}_t}{\wb{q}}\right)^2 = R
\end{align*}
To compute $\mu$ we derive
\begin{align*}
\expectedvalue[ \:V_s ]
&= \expectedvalue_{1:M_i}\left[ \sum_{k=1}^{M_i}k^2\indfunc\{b_{i,k+1} = 0 \cap b_{i,k} \neq 0\}\:\psi_{i}\:\psi_{i}^\transp\:\psi_{i}\:\psi_{i}^\transp\right ]\\
&= \sum_{k=1}^{M_i}\expectedvalue_{1:k}\left[ k^2\indfunc\{b_{i,k+1} = 0 \cap b_{i,k} \neq 0\}\:\psi_{i}\:\psi_{i}^\transp\:\psi_{i}\:\psi_{i}^\transp\right ]
= \sum_{k=1}^{M_i} k^2 \probability(b_{i,k+1} = 0) \probability(b_{i,k} \neq 0)\:\psi_{i}\:\psi_{i}^\transp\:\psi_{i}\:\psi_{i}^\transp\\
&= \sum_{k=1}^{M_i} k^2 \left(1 - \frac{k}{k+1}\right) \frac{1}{k}\:\psi_{i}\:\psi_{i}^\transp\:\psi_{i}\:\psi_{i}^\transp
= \sum_{k=1}^{M_i} k \left(\frac{k + 1 -k}{k+1}\right) \:\psi_{i}\:\psi_{i}^\transp\:\psi_{i}\:\psi_{i}^\transp
= \sum_{k=1}^{M_i} \frac{k}{k+1} \:\psi_{i}\:\psi_{i}^\transp\:\psi_{i}\:\psi_{i}^\transp\\
&\preceq \sum_{k=1}^{M_i} \:\psi_{i}\:\psi_{i}^\transp\:\psi_{i}\:\psi_{i}^\transp
= M_i \:\psi_{i}\:\psi_{i}^\transp\:\psi_{i}\:\psi_{i}^\transp
\preceq \frac{2\alpha\beta\deff{\gamma}_t}{\tau_{i,t}\wb{q}} \:\psi_{i}\:\psi_{i}^\transp\:\psi_{i}\:\psi_{i}^\transp
\preceq \frac{2\alpha\beta\deff{\gamma}_t}{\wb{q}} \:\psi_{i}\:\psi_{i}^\transp
\end{align*}
Therefore
 \begin{align*}
  \mu = \norm{\sum\nolimits_{i=1}^t{\expectedvalue[\:V_i]} }{2}
  \leq\norm{\sum\nolimits_{j}\frac{2\alpha\beta\deff{\gamma}_t}{\wb{q}}\:\psi_{i}\:\psi_{i}^\transp }{2}
  =\norm{\frac{2\alpha\beta\deff{\gamma}_t}{\wb{q}}\:\Psi\:\Psi^\transp }{2}
  \leq\frac{2\alpha\beta\deff{\gamma}_t}{\wb{q}}
 \end{align*}
We can now apply Proposition \ref{prop:matrix-chernoff} with $\mu$, $R = \mu^2$
and $\delta = 2$.
\begin{align*}
&\probability\left(\lambda_{\max}\left(\sum\nolimits_{j}{\:X_{j}}\right)\geq 3\frac{2\alpha\beta\deff{\gamma}_t}{\wb{q}}\right)
\leq \probability\left(\lambda_{\max}\left(\sum\nolimits_{j}{\:X_{j}}\right)\geq (1+\delta)\mu\right)\\
&\leq t\left( \frac{e^2}{27} \right)^{1/\mu}
\leq t\exp\left\{  -\frac{1}{\mu}(\log(27) - 2)\right\}
\leq t\exp\left\{  -\frac{\wb{q}}{2\alpha\beta\deff{\gamma}_t}\right\}.
\end{align*}
Therefore with high probability we have
\begin{align*}
\lambda_{\max}(\:W_{t})
\leq 
\lambda_{\max}\left(\sum_{i=1}^t \:V_{i}\right) \leq 3\frac{2\alpha\beta\deff{\gamma}_t}{\wb{q}} = \sigma^2
\end{align*}
Plugging this in the Freedman bound
\begin{align*}
&\probability\left( \lambda_{\max}(\:Y_{t}) \geq \varepsilon \cap \lambda_{\max}(\:W_{t}) \leq 3\frac{2\alpha\beta\deff{\gamma}_t}{\wb{q}}\right)\\
	&\leq t \cdot \exp \left\{ - \frac{ -\varepsilon^2/2 }{ (3 + \varepsilon/3)\frac{2\alpha\beta\deff{\gamma}_t}{\wb{q}} }\right\}
	= t \cdot \exp \left\{ - \frac{ \varepsilon^2/2 }{ (3 + \varepsilon/3) }\frac{\wb{q}}{2\alpha\beta\deff{\gamma}_t}\right\}.
\end{align*}
Therefore, by carefully choosing $\wb{q}$ we have that
\begin{align*}
&\probability\left( \lambda_{\max}(\wh{\:Y}_{t}) \geq \varepsilon\right)
\leq \probability\left( \lambda_{\max}(\:Y_{t}) \geq \varepsilon\right)\\
&\leq \probability\left( \lambda_{\max}(\:Y_{t}) \geq \varepsilon \cap \lambda_{\max}(\:W_{t}) \leq \sigma^2\right)
+ \probability\left( \lambda_{\max}(\:W_{t}) > \sigma^2\right)
\leq \frac{\delta}{4t} + \frac{\delta}{4t} \leq \frac{\delta}{2t}
\end{align*}

\paragraph{Space complexity}
We must now separately bound the number of columns present in the dictionary
at each time step. This is equivalent, at time $t$,
to counting how many $b_{i,t}$ are different than 0.
Again, from the terminating condition in the algorithm, we know that
\begin{align*}
    \probability(b_{i,t} \neq 0) = 1/b_{i,t} \leq \wb{q}\wt{p}_{i,t} \leq \wb{q}p_{i,t}
\end{align*}
Let $z_{i,t} = \indfunc\{b_{i,t} \neq 0\}$ be the random variable
that indicates whether we column $i$ survived until time $t$ or not.
Notice the trial of all columns are independent. By definition $z_{i,t}$ are
Bernoulli random variables and their probability parameter is upper bounded by $\wb{p}_{i,t} = \min\{\wb{q}p_{i,t},1\}$. We can rewrite the total number of columns kept as
    \begin{align*}
        Q_t = \sum_{i=1}^{t} z_{i,t}
    \end{align*}
    First, it is easy to see that $\expectedvalue[\sum_{i=1}^{t} z_{i}] \leq \wb{q}$.
     Using Hoeffding bound we can show
    \begin{align*}
        &\probability\left( \sum_{i=1}^t z_i \geq g\wb{q} \right) = \inf_\theta\probability\left( e^{ \sum_{i=1}^t \theta z_i} \geq e^{\theta g \wb{q}} \right)\\
        &\leq \inf_\theta \frac{\expectedvalue \left[ e^{ \sum_{i=1}^t \theta z_i}\right]}{e^{\theta g\wb{q}}}
        = \inf_\theta \frac{\expectedvalue \left[  \prod_{i=1}^t e^{ \theta z_i}\right]}{e^{\theta g\wb{q}}}
        = \inf_\theta \frac{ \prod_{i=1}^t \expectedvalue \left[ e^{ \theta z_i}\right]}{e^{\theta g\wb{q}}}\\
        &=\inf_\theta \frac{ \prod_{i=1}^t (\wb{p}_{i,t} e^\theta + (1-\wb{p}_{i,t}))}{e^{\theta g\wb{q}}}
        = \inf_\theta \frac{ \prod_{i=1}^t (1+\wb{p}_{i,t}( e^\theta -1))}{e^{\theta g\wb{q}}}
        \leq \inf_\theta \frac{ \prod_{i=1}^t e^{\wb{p}_{i,t}( e^\theta -1)}}{e^{\theta g\wb{q}}}\\
        &\leq \inf_\theta \frac{e^{\wb{q}( e^\theta -1)}}{e^{\theta g\wb{q}}}
        = \inf_\theta e^{(\wb{q} e^\theta -\wb{q} - \theta g\wb{q})},
    \end{align*}
where we use the fact that $\sum_{i=1}^t p_{i,t} \leq 1$.
The choice of $\theta$ minimizing the previous expression is obtained as
    \begin{align*}
        \frac{d}{d\theta}e^{( \wb{q} e^\theta -\wb{q} - \theta g\wb{q})}
        = e^{( \wb{q}e^\theta -\wb{q} - \theta g\wb{q})}(\wb{q}e^\theta - g \wb{q}) = 0,
    \end{align*}
    and thus $\theta = \log (g)$. Finally
    \begin{align*}
        &\probability\left(  \sum_{i=1}^t z_i \geq g\wb{q} \right)
        \leq \inf_\theta e^{ \wb{q}(e^\theta -1 - \theta g)}
        = e^{ \wb{q}(g  -1 -  g \log g)} \leq e^{-\wb{q}g   (\log g - 1)}.
    \end{align*}
    Taking $g = e^{2}$, we have $\log(g) = 2$ and that gives us
    \begin{align*}
        \probability\left(  \sum_{i=1}^t z_i \geq e^2\wb{q} \right)
        \leq& \exp\left\{-e^2\wb{q}(2 - 1)\right\}\\
        \leq& \exp\left\{-e^2\frac{28\alpha\beta\deff{\gamma}_{t}}{\varepsilon^{2}}\log\left(\frac{4t}{\delta}\right)\right\} = \left(\frac{\delta}{4t}\right)^{\left(e^2\frac{28\alpha\beta\deff{\gamma}_{t}}{\varepsilon^{2}}\right)} \leq \frac{\delta}{2t}.
    \end{align*}
    Therefore
    \begin{align*}
        \probability\left(\norm{\:\Psi\:\Psi^\transp - \:\Psi\selmatrix_t \selmatrix_t^\transp\:\Psi^\transp}{2} \geq \varepsilon \cup Q_t \geq 8\wb{q} \right) \leq \frac{\delta}{2t} + \frac{\delta}{2t} \leq \frac{\delta}{t},
    \end{align*}
    and this concludes the proof.

\end{proof}
\begin{proof}[Proof or Corollary \ref{cor:our-kern-reg-general}]
    To simplify the proof, we first introduce the quantities $\gamma = t\gamma'$
    and $\mu = t\mu'$ and $\gamma/\mu = \gamma'/\mu'$.
    We begin by decomposing the generalization error in a bias and variance part,
    \begin{align*}
&\mathcal{R}(\wt{\:w}_t)=\mathrm{E}_{\psi}\Vert \akermatrix_t(\akermatrix_t+t\mu' I)^{-1}(\:f^{*}+\sigma^{2}\xi)-\:f^{*}\Vert_{2}^{2}\\
&=\Vert(\akermatrix_t(\akermatrix_t+t\mu' \:I)^{-1} - \:I)\:f^{*}\Vert_{2}^{2}+\sigma^{2}\mathrm{E}_{\psi}\Vert \akermatrix_t(\akermatrix_t+t\mu' \:I)^{-1}\xi\Vert_{2}^{2}\\
&=t^2\mu'^{2}\Vert(\akermatrix_t+t\mu' \:I)^{-1}\:f^{*}\Vert_{2}^{2} + \sigma^{2} \Tr (\akermatrix_t^{2}(\akermatrix_t+t\mu' \:I)^{-2})\\
&:= \mbox{bias} (\akermatrix_t)^{2} + \mbox{variance} (\akermatrix_t)
    \end{align*}
    From the proof of \citet[App. A, Lemma 2]{alaoui2014fast}, we have
    \begin{align*}
        \Vert(\akermatrix_t+t\mu' \:I)^{-1}\:f^{*}\Vert_{2}
        \leq& \Vert(\kermatrix_t+t\mu' \:I)^{-1}\:f^{*}\Vert_{2}\left( 1 + \frac{t\gamma'}{1-\varepsilon}\Vert(\akermatrix_t+t\mu' \:I)^{-1}\Vert_{2}\right)\\
        \leq& \Vert(\kermatrix_t+t\mu' \:I)^{-1}\:f^{*}\Vert_{2}\left( 1 + \frac{\gamma'/\mu'}{1-\varepsilon}\right)
    \end{align*}
    Therefore
    \begin{align*}
         \mbox{bias} (\akermatrix_t)^{2} \leq \left( 1 + \frac{\gamma'/\mu'}{1-\varepsilon}\right)^{2}\mbox{bias}(\kermatrix_t)^2
    \end{align*}
    It is easy to see that the variance decreases if we use $\akermatrix_t$
    instead of $\kermatrix_t$.
    We can rewrite the variance as
    \begin{align*}
        \mbox{variance} (\kermatrix_t) = \Tr(\kermatrix_t^2(\kermatrix_t + t\mu\:I)^{-2}) = \sum_{i=1}^{t} \frac{\lambda_i^{2}}{(\lambda_{i} + t\mu')^{2}},
    \end{align*}
    where $\lambda_i$ are the eigenvalues of the kernel matrix and it shows that the variace is strictly increasing in $\lambda_{i}$. Because
    $\akermatrix_t \preceq \kermatrix_t$, the same ordering applies to each eigenvalue
    of the two matrices, and therefore
    \begin{align*}
        \mbox{variance}(\akermatrix_t) \leq \mbox{variance}(\kermatrix_t).
    \end{align*}
    Putting it all together
    \begin{align*}
\mathcal{R}(\wt{\:w}_t)=& \mbox{bias} (\akermatrix_t)^{2} + \mbox{variance} (\akermatrix_t)\\
\leq& \left( 1 + \frac{\gamma'/\mu'}{1-\varepsilon}\right)^{2}\mbox{bias} (\kermatrix_t)^{2} + \mbox{variance} (\kermatrix_t)\\
\leq& \left( 1 + \frac{\gamma'/\mu'}{1-\varepsilon}\right)^{2}(\mbox{bias} (\kermatrix_t)^{2} + \mbox{variance} (\kermatrix_t)) = \left( 1 + \frac{\gamma'/\mu'}{1-\varepsilon}\right)^{2}\mathcal{R}(\wh{\:w}_t)\\
=& \left( 1 + \frac{\gamma/\mu}{1-\varepsilon}\right)^{2}\mathcal{R}(\wh{\:w}_t)
    \end{align*}

\end{proof}

\section{Extended proofs of Section \ref{sec:incremental-oracle}}\label{sec:app:incremental-oracle}

\begin{proof}[Proof of Lemma \ref{lem:fast-rls}]
    Subtracting $\kermatrix_{t+1} - \overline{\kermatrix}_{t+1}$ and recalling the block matrix multiplication formula we have
    \begin{align*}
        &\left[
    \begin{array}{c|c}
        \:x_{t+1}^\transp & x_{t+1}
\end{array}\right]
        \left[
    \begin{array}{c|c}
        \kermatrix_t - \akermatrix_{t}& \:0 \\
    \hline
    \:0^\transp & 0
    \end{array}
    \right] 
        \left[
    \begin{array}{c}
        \:x_{t+1} \\
        \hline
        x_{t+1}
    \end{array}
\right] = \:x_{t+1}^\transp (\kermatrix_t - \bkermatrix) \:x_{t+1} \leq \frac{\gamma}{1-\varepsilon}\:x_{t+1}^\transp \:x_{t+1}
\end{align*}
Therefore, $\bkermatrix_{t+1}$ satisfies  $\eqref{cond:alaoui-nyst-guar}$.
For the remainder of the proof we drop the dependency on time $t+1$,
and simply write $\kermatrix$ and $\bkermatrix$.
Let $\eta = \frac{2-\varepsilon}{1-\varepsilon}$, from Proposition~\ref{prop:alaoui-nyst-guar}, we derive
    \begin{align*}
    (\kermatrix + \gamma\:I)^{-1}  \succeq (\bkermatrix + \eta\gamma\:I)^{-1} \succeq \left(\kermatrix + \eta\gamma\:I\right)^{-1} \succeq \frac{1}{\eta}(\kermatrix + \gamma\:I)^{-1}
    \end{align*}
    We need to prove something along the lines of
\begin{align*}
    \tau_i =& \:k_i^\transp \left(\kermatrix + \gamma\:I\right)^{-1} \:e_i =\:e_i^\transp \kermatrix \left(\kermatrix + \gamma\:I\right)^{-1} \:e_i = \:e_i^\transp \kermatrix^{1/2} \left(\kermatrix + \gamma\:I\right)^{-1}\kermatrix^{1/2} \:e_i\\
    &\geq \:e_i^\transp \kermatrix^{1/2} \left(\bkermatrix + \eta\gamma\:I\right)^{-1}\kermatrix^{1/2} \:e_i \geq \frac{1}{\eta}\:e_i^\transp \kermatrix^{1/2} \left(\kermatrix + \gamma\:I\right)^{-1}\kermatrix^{1/2} \:e_i = \frac{1}{\eta}\tau_i
\end{align*}
where the middle line would be our estimator.
The problem is that we do not have access to $\kermatrix^{1/2}$ (it would
take too much time and space to compute), and we do not want our bound
to depend on the smallest eigenvalue.

We can proceed as follows. Differently from \citet{alaoui2014fast} we
will only look to approximate leverage scores for columns in $\selmatrix$,
or in other words only entries strictly on the diagonal
of $\kermatrix(\kermatrix + \gamma)^{-1}$, and only for columns that we fully
store.

\noindent We begin by noting

\begin{align*}
    &\tau_i = \:k_i(\kermatrix + \gamma\:I)^{-1}\:e_i = \frac{1}{\gamma}\:e_i \kermatrix(\kermatrix + \gamma\:I)^{-1}(\gamma\:I)\:e_i =  \frac{1}{\gamma}\:e_i \kermatrix(\kermatrix + \gamma\:I)^{-1}(\kermatrix -\kermatrix + \gamma\:I)\:e_i\\
    &  \frac{1}{\gamma}\:e_i \kermatrix(\kermatrix + \gamma\:I)^{-1}(\kermatrix + \gamma\:I)\:e_i - \frac{1}{\gamma}\:e_i \kermatrix(\kermatrix + \gamma\:I)^{-1}\kermatrix\:e_i = \frac{1}{\gamma}(k_{i,i} - \:k_i (\kermatrix + \gamma\:I)^{-1}\:k_i)
\end{align*}

\noindent We can easily see that
\begin{align*}
    \tau_i =& \frac{1}{\gamma}(k_{i,i} - \:k_i (\kermatrix + \gamma\:I)^{-1}\:k_i) \\
    &\leq \frac{1}{\gamma}\left(k_{i,i} - \:k_i \left(\bkermatrix + \eta\gamma\:I\right)^{-1}\:k_i\right)
    \leq \frac{1}{\gamma}\left(k_{i,i} - \:k_i \left(\kermatrix + \eta\gamma\:I\right)^{-1}\:k_i\right)
\end{align*}
Now
\begin{align*}
&k_{i,i} - \:k_i (\kermatrix + \eta\gamma\:I)^{-1}\:k_i
= \:e_i^\transp \kermatrix \:e_i - \:e_i^\transp \kermatrix\left(\kermatrix + \eta\gamma\:I\right)^{-1}\kermatrix\:e_i\\
&= \:e_i^\transp \kermatrix\left(\:I - \kermatrix\left(\kermatrix + \eta\gamma\:I\right)^{-1}\right)\:e_i
= \eta\gamma\:e_i^\transp \kermatrix\left(\kermatrix + \eta\gamma\:I\right)^{-1}\:e_i\\
&\leq \eta\gamma\:e_i^\transp \kermatrix(\kermatrix + \gamma\:I)^{-1}\:e_i = \eta\gamma\tau
\end{align*}
Putting it all together
    \begin{align*}
    \tau \leq \frac{1}{\gamma}\left(k_{i,i} - \:k_i \left(\bkermatrix + \eta\gamma\:I\right)^{-1}\:k_i\right)\leq \frac{1}{\gamma}\left(k_{i,i} - \:k_i \left(\kermatrix + \eta\gamma\:I\right)^{-1}\:k_i\right) \leq \eta\frac{\gamma}{\gamma}\tau
    \end{align*}
\end{proof}


\begin{proof}[Proof of Lemma \ref{lem:fast-deff}]
    This proof proceeds in two steps. First, we will find upper and lower bounds
    for the term reported in Equation~\ref{eq:deff-increase-exact}. Then we will use induction, and the
    fact that we can compute $\deff{\gamma}_0$ exactly to prove the claim.
    Let $\eta = \frac{2-\varepsilon}{1-\varepsilon}$
    We begin by reminding that as a consequence of Proposition~\ref{prop:alaoui-nyst-guar},
    and of properties of nonsingular PSD matrices,
    we have
    \begin{align*}
    (\kermatrix + \gamma\:I)^{-1}  \succeq \left(\akermatrix_t + \eta\gamma_{\varepsilon}\:I\right)^{-1} \succeq \left(\kermatrix + \eta\gamma\:I\right)^{-1} \succeq \frac{1}{\eta}(\kermatrix + \gamma\:I)^{-1}
    \end{align*}
    We remind that $\kermatrix_t = \featkermatrix_t^\transp\featkermatrix_t = \:U\:\Sigma^\transp\:\Sigma\:U^\transp$
    and $\wb{\:k}_{t+1} = \featkermatrix_t^\transp\:\phi$
    with $\featkermatrix_t = \:V\:\Sigma\:U^\transp$.
    Similarly, we introduce $\akermatrix_t = \wt{\:U}\wt{\:\Sigma}^\transp\wt{\:\Sigma}\wt{\:U}^\transp$
    we have now
    \begin{align*}
        (\:\Sigma^\transp\:\Sigma + \gamma\:I)^{-1} \succeq \:U^\transp\wt{\:U}\left(\wt{\:\Sigma}^\transp\wt{\:\Sigma} + \eta\gamma\:I\right)^{-1}\wt{\:U}^\transp\:U \succeq \left(\:\Sigma^\transp\:\Sigma +\eta\gamma\:I\right)^{-1}
    \end{align*}
    We need a bound on
    \begin{align*}
        & \wt{\Delta}_t = \frac{\left( k_{t+1} - 
                \wb{\:k}_{t+1}^\transp \left(\akermatrix_{t}^{\gamma} + \eta\gamma\:I\right)^{-1}\wb{\:k}_{t+1}
        - \frac{(1-\varepsilon)^2}{4}\gamma\wb{\:k}_{t+1}^\transp (\akermatrix_{t}^{\gamma} + \gamma\:I)^{-2}\wb{\:k}_{t+1}\right)
}{\frac{1}{\eta}(k_{t+1} +\gamma - \wb{\:k}_{t+1}^\transp \left(\akermatrix_{t}^{\gamma} + \eta\gamma\:I\right)^{-1}\wb{\:k}_{t+1})}
    \end{align*}
    We will first upper and lower bound the denumerator
    \begin{align*}
        &k_{t+1} +\gamma - \wb{\:k}_{t+1}^\transp (\kermatrix_{t} + \gamma\:I)^{-1}\wb{\:k}_{t+1}\\
        &\leq k_{t+1} +\gamma - \wb{\:k}_{t+1}^\transp \left(\akermatrix_{t}^{\gamma} +\eta\gamma\:I\right)^{-1}\wb{\:k}_{t+1}\\
        &= \gamma + \phi^\transp \:V\:V^\transp \phi  - \:\phi^\transp \:V \:\Sigma \:U^\transp\wt{\:U}\left(\wt{\:\Sigma}^\transp\wt{\:\Sigma} + \eta\gamma\:I\right)^{-1}\wt{\:U}^\transp\:U\:\Sigma^\transp\:V^\transp\:\phi\\
        &\leq \gamma + \phi^\transp \:V\:V^\transp \phi  - \:\phi^\transp \:V \:\Sigma \left(\:\Sigma^\transp\:\Sigma + \eta\gamma\:I\right)^{-1}\:\Sigma^\transp\:V^\transp\:\phi\\
        &= \gamma + \:\phi^\transp \:V \left(\:I - \:\Sigma \left(\:\Sigma^\transp\:\Sigma + \eta\gamma\:I\right)^{-1}\:\Sigma^\transp\right)\:V^\transp\:\phi
    \end{align*}
    The last expression corresponds to a block diagonal matrix, where
    for all $i>t$, only the diagonal remains, with 1 on the diagonal, that we can easily upper bound with 2, because $1<2$
    for all choiches of 1 and 2.
    We want to study the $i$-th entry in the diagonal matrix
    \begin{align*}
        1-\frac{\sigma^{2}}{\sigma^{2} + \eta\gamma} = \eta\frac{\gamma}{\sigma^{2} + \eta\gamma}\leq  \eta\frac{\gamma}{\sigma_{2} + \gamma} =\eta\left(1-\frac{\sigma^{2}}{\sigma^{2} + \gamma}\right)
    \end{align*}
    Thus,
    \begin{align*}
        &\gamma + \:\phi^\transp \:V \left(\:I - \:\Sigma \left(\:\Sigma^\transp\:\Sigma + \eta\gamma\:I\right)^{-1}\:\Sigma^\transp\right)\:V^\transp\:\phi\\
        & \leq \gamma + \eta\:\phi^\transp \:V (\:I - \:\Sigma (\:\Sigma^\transp\:\Sigma + \gamma\:I)^{-1}\:\Sigma^\transp)\:V^\transp\:\phi\\
        & \leq \eta(\gamma + \:\phi^\transp \:V (\:I - \:\Sigma (\:\Sigma^\transp\:\Sigma + \gamma\:I)^{-1}\:\Sigma^\transp)\:V^\transp\:\phi)\\
        & = \eta(k_{t+1} +\gamma - \wb{\:k}_{t+1}^\transp (\kermatrix_{t} + \gamma\:I)^{-1}\wb{\:k}_{t+1})
    \end{align*}
    We should now bound the numerator.
    We will also need bounds on the squares of the inequalities we used this far. In particular we will
    begin with the one we have for free,
    \begin{align*}
        (\kermatrix_t - \akermatrix_{t})^{2} \preceq \frac{\gamma^{2}}{(1-\varepsilon)^{2}}\:I,
    \end{align*}
    because $\:I$ commutes with everything.
    Given PD matrices $\:A$ and $\:B$ we have $\:A^2 \geq \:B^2$ if and only if
    the largest singular value of $\:A^{-1}\:B$ is smaller or equal than 1.
    We begin with
    \begin{align*}
        (\kermatrix_t + \gamma\:I)^2 \preceq \left(\frac{2}{1-\varepsilon}(\akermatrix_t + \gamma\:I)\right)^2.
    \end{align*}
    The singular values of matrix $\:A$ are the square root of the eigenvalues
    of $\:A^\transp\:A$ or $\:A\:A^\transp$. They can be also defined
    as $\normsmall{\:A}_{2}$. In our case, we want to show
    \begin{align*}
        \norm{(\kermatrix_t + \gamma\:I)\left(\frac{2}{1-\varepsilon}(\akermatrix_t + \gamma\:I)\right)^{-1}}{2} \leq 1.
    \end{align*}
    We have
    \begin{align*}
        &\norm{(\kermatrix_t + \gamma\:I)\left(\frac{2}{1-\varepsilon}(\akermatrix_t + \gamma\:I)\right)^{-1}}{2}
        = \norm{(\kermatrix_t - \akermatrix_t + \akermatrix_t + \gamma\:I)\left(\frac{2}{1-\varepsilon}(\akermatrix_t + \gamma\:I)\right)^{-1}}{2}\\
        &\leq \norm{(\kermatrix_t - \akermatrix)\left(\frac{2}{1-\varepsilon}(\akermatrix_t + \gamma\:I)\right)^{-1}}{2} + \norm{(\akermatrix_t + \gamma\:I)\left(\frac{2}{1-\varepsilon}(\akermatrix_t + \gamma\:I)\right)^{-1}}{2}
    \end{align*}
    Because $(\kermatrix_t - \akermatrix)^{2} \preceq \gamma^{2}/(1-\varepsilon)^2\:I$
    holds for the squares, we can use it on the first term and obtain
    \begin{align*}
        \norm{(\kermatrix_t + \gamma\:I)\left(\frac{2}{1-\varepsilon}(\akermatrix_t + \gamma\:I)\right)^{-1}}{2} \leq  \frac{(1-\varepsilon)}{2}\left(\max_{i}\frac{\gamma}{(1-\varepsilon)(\wt{\lambda}_{i} + \gamma)} + \max_{i}\frac{\wt{\lambda}_{i} + \gamma}{\wt{\lambda}_{i} + \gamma}\right) \leq \frac{1}{2} {2}
    \end{align*}
    Similarly for
    \begin{align*}
        (\akermatrix_t + \gamma\:I)^2 \preceq \left(\frac{2}{1-\varepsilon}(\kermatrix_t + \gamma\:I)\right)^2 
    \end{align*}
    We have
    \begin{align*}
        &\norm{(\akermatrix_t + \gamma\:I)\left(\frac{2}{1-\varepsilon}(\kermatrix_t + \gamma\:I)\right)^{-1}}{2} = \norm{(\akermatrix_t - \kermatrix_t  + \kermatrix + \gamma\:I)\left(\frac{2}{1-\varepsilon}(\kermatrix_t + \gamma\:I)\right)^{-1}}{2}\\
        &\leq \norm{(\akermatrix_t - \kermatrix_t)\left(\frac{2}{1-\varepsilon}(\kermatrix_t + \gamma\:I)\right)^{-1}}{2} + \norm{(\kermatrix + \gamma\:I)\left(\frac{2}{1-\varepsilon}(\kermatrix_t + \gamma\:I)\right)^{-1}}{2}
    \end{align*}
    And
    \begin{align*}
        \norm{(\akermatrix_t + \gamma\:I)\left(\frac{2}{1-\varepsilon}(\kermatrix_t + \gamma\:I)\right)^{-1}}{2} \leq  \frac{1-\varepsilon}{2}\left(\max_{i}\frac{\gamma}{(1-\varepsilon)(\lambda_{i} + \gamma)} + \max_{i}\frac{\lambda_i+\gamma}{\lambda_{i} + \gamma}\right) \leq \frac{1}{2} {2}
    \end{align*}
    Therefore
    \begin{align*}
        \frac{(1-\varepsilon)^{4}}{16}(\kermatrix_t + \gamma\:I)^{-2} \preceq  \frac{(1-\varepsilon)^{2}}{4}(\akermatrix_t + \gamma\:I)^{-2} \preceq (\kermatrix_t + \gamma\:I)^{-2}
    \end{align*}
    Similarly to the denominator, we derive
    \begin{align*}
        &k_{t+1} - \wb{\:k}_{t+1}^\transp (\kermatrix_{t} + \gamma\:I)^{-1}\wb{\:k}_{t+1}- \gamma\wb{\:k}_{t+1}^\transp (\kermatrix_{t} + \gamma\:I)^{-2}\wb{\:k}_{t+1}\\
        \leq& k_{t+1} - \wb{\:k}_{t+1}^\transp \left(\akermatrix_{t}^{\gamma} + \eta\gamma\:I\right)^{-1}\wb{\:k}_{t+1}- \frac{(1-\varepsilon)^2}{4}\gamma\wb{\:k}_{t+1}^\transp (\akermatrix_{t}^{\gamma} + \gamma\:I)^{-2}\wb{\:k}_{t+1}\\
        =& \:\phi^\transp \:V\:V^\transp \:\phi -
        \:\phi^\transp \:V \:\Sigma \:U^\transp\wt{\:U}(\wt{\:\Sigma}^\transp\wt{\:\Sigma} + \eta\gamma\:I)^{-1}\wt{\:U}^\transp\:U\:\Sigma^\transp\:V^\transp\:\phi
        - \frac{(1-\varepsilon)^2}{4}\gamma\:\phi^\transp \:V \:\Sigma \:U^\transp\wt{\:U}(\wt{\:\Sigma}^\transp\wt{\:\Sigma} + \gamma\:I)^{-2}\wt{\:U}^\transp\:U\:\Sigma^\transp\:V^\transp\:\phi\\
        \leq& \:\phi^\transp \:V\:V^\transp \:\phi - \:\phi^\transp \:V \:\Sigma \left(\:\Sigma^\transp\:\Sigma + \eta\gamma\:I\right)^{-1}\:\Sigma^\transp\:V^\transp\:\phi
        - \frac{(1-\varepsilon)^4}{16}\gamma\:\phi^\transp \:V \:\Sigma (\:\Sigma^\transp\:\Sigma + \gamma\:I)^{-2}\:\Sigma^\transp\:V^\transp\:\phi\\
        =& \:\phi^\transp \:V \left(\:I - \:\Sigma \left(\:\Sigma^\transp\:\Sigma + \eta\gamma\:I\right)^{-1}\:\Sigma^\transp - \frac{(1-\varepsilon)^4}{16}\gamma\:\Sigma (\:\Sigma^\transp\:\Sigma + \gamma\:I)^{-2}\:\Sigma^\transp\right)\:V^\transp\:\phi
    \end{align*}
    Again, we study the object
    \begin{align*}
        &1-\frac{\sigma^{2}}{\sigma^{2} + \eta\gamma} - \frac{(1-\varepsilon)^4}{16}\frac{\gamma\sigma^{2}}{(\sigma^{2} + \gamma)^2}
        \leq 1-\frac{\sigma^{2}}{\sigma^{2} + \eta\gamma} - \frac{(1-\varepsilon)^4}{16}\frac{\gamma\sigma^{2}}{(\sigma^{2} + \eta\gamma)(\sigma^{2} + \gamma)}\\
        &=\frac{\sigma^4 +(1+\eta)\gamma\sigma^2 +\eta\gamma^2 - \sigma^4 -\gamma\sigma^{2} - ((1-\varepsilon)^{4}/16)\gamma\sigma^{2}}{(\sigma^{2} + \eta\gamma)(\sigma^{2} + \gamma)}\\
        &\leq\frac{\eta\gamma\sigma^{2} +\eta\gamma^2}{(\sigma^{2} + \gamma)^2}
        \leq \eta\left(1+ \frac{\sigma^2}{\gamma}\right)\frac{\gamma^2}{(\sigma^{2} + \gamma)^2}
        \leq \eta\left(1+ \frac{\lambda_{\max}}{\gamma}\right)\frac{\gamma^2}{(\sigma^{2} + \gamma)^2}
    \end{align*}
    Therefore
    \begin{align*}
        \Delta_{t} \leq& \wt{\Delta}_{t} \leq \eta^2\left(1+ \frac{\lambda_{\max}}{\gamma}\right)\Delta_{t}
    \end{align*}
    Now assume the inductive hypotheses that $\adeff{\gamma}_t$ is $\beta$-approximate
    holds, with $\beta = \eta^{2}\left(1+ \frac{\lambda_{\max}}{\gamma}\right)$,
    and we can see that
    \begin{align*}
        &\adeff{\gamma}_{t+1} = \adeff{\gamma}_{t} + \wt{\Delta}_t \leq \beta(\deff{\gamma}_{t} + {\Delta}_t) = \beta\deff{\gamma}_{t+1}\\
        &\adeff{\gamma}_{t+1} = \adeff{\gamma}_{t} + \wt{\Delta}_t \geq \deff{\gamma}_{t} + {\Delta}_t= \deff{\gamma}_{t+1}
    \end{align*}
    Therefore $\adeff{\gamma}_{t+1}$ is also a $\beta$-approximation. Because
    at the first iteration we can compute $\adeff{\gamma}_{0}$ exactly,
    we can prove Lemma \ref{lem:fast-deff} by induction.
\end{proof}


\end{document}